\documentclass[10pt,twoside]{article}
\usepackage{amsmath}
\usepackage[utf8]{inputenc}
\usepackage{url}
\usepackage{dsfont}
\usepackage{graphicx}
\usepackage{yhmath}
\usepackage[table]{xcolor}
\usepackage{mathrsfs} 
\usepackage{amssymb}
\usepackage{url}
\usepackage{makecell}
\usepackage{arydshln}
\usepackage{multirow}
\usepackage{arydshln}
\usepackage{mathtools}
\usepackage{stmaryrd}  
 \usepackage{amsthm,amssymb}
\usepackage{hyperref}
\usepackage{bm} 
\usepackage{dsfont}

\input epsf
\numberwithin{equation}{section}

\newtheorem{thm}{Theorem}[section]

\newtheorem{example}[thm]{Example}

\newtheorem{lem}[thm]{Lemma}

\newtheorem{rmrk}[thm]{Remark}

\newenvironment{dem}{\ \\ {\bf Proof. }}
{\qed\par\medskip\noindent}
\newcommand{\E}{\ensuremath{\mathbb{E}}}
\newcommand{\R}{\ensuremath{\mathbb{R}}}
\newcommand{\Z}{\ensuremath{\mathbb{Z}}}
\newcommand{\C}{\ensuremath{\mathbb{C}}}
\newcommand{\N}{\ensuremath{\mathbb{N}}}

\newcommand{\lip}{\ensuremath{\mathrm{Lip}}}
\definecolor{grisclair}{gray}{0.9}
\font\dsrom=dsrom10 scaled 1200
\def \ind{\textrm{\dsrom{1}}}
\DeclareMathOperator*{\argmin}{argmin}

 \newcommand{\dx}{ {d_x} }

\newcommand{\dy}{ {d_y} }

\newcommand{\mk}{ { \mathcal{K}} }
\newcommand{\mx}{ { \mathcal{X}} }
\newcommand{\my}{ { \mathcal{Y}} }
\newcommand{\mz}{ { \mathcal{Z}} }
\newcommand{\mf}{ { \mathcal{F}} }
\newcommand{\tmix}{ { t_{\text{mix}}} }

\newcommand{\KL}{ { \text{KL} } }

\begin{document}
\title{\bf Minimax optimality of deep neural networks on dependent data via PAC-Bayes bounds}
 \maketitle \vspace{-1.0cm}
 %

  \begin{center}
   Pierre Alquier$^{\text{a}}$  and
   William Kengne$^{\text{b}}$
 \end{center}

  \begin{center}
{\it $^{\text{a}}$  ESSEC Business School, Asia-Pacific campus, Singapore \\
 $^{\text{b}}$ Universit\'{e} Jean Monnet, ICJ UMR5208, CNRS, Ecole Centrale de Lyon, \\
 INSA Lyon, Universit\'{e} Claude Bernard Lyon 1, Saint-\'{E}tienne, France \\ 
    E-mail:  alquier@essec.edu ; william.kengne@univ-st-etienne.fr \\
  }
\end{center}

 \pagestyle{myheadings}
 \markboth{PAC-Bayes deep neural networks estimator}{Alquier and Kengne}

~~\\
\textbf{Abstract}:
 In a groundbreaking work,~\cite{schmidt2020nonparametric} proved the minimax optimality of deep neural networks with ReLU activation for least-square regression estimation over a large class of functions defined by composition. In this paper, we extend these results in many directions. First, we remove the i.i.d. assumption on the observations,  to allow some time dependence. The observations are assumed to be a Markov chain with a non-null pseudo-spectral gap. Then, we study a more general class of machine learning problems, which includes least-square and logistic regression as special cases. Leveraging on PAC-Bayes oracle inequalities and a version of Bernstein inequality due to~\cite{paulin2015concentration}, we derive upper bounds on the estimation risk for a generalized Bayesian estimator. In the case of least-square regression, this bound matches (up to a logarithmic factor) the lower bound in \cite{schmidt2020nonparametric}. We establish a similar lower bound for classification with the logistic loss, and prove that the proposed DNN estimator is optimal in the minimax sense.

 \medskip
 
 {\em Keywords:} Deep neural networks, minimax-optimality, Bayesian neural network, PAC-Bayes bounds, oracle inequality.

\section{Introduction}\label{sect_intro}
Since almost 20 years, deep neural networks (DNN) are the most efficient predictors for most large-scale applications of machine learning, including images, sound and video processing, high-dimensional time-series prediction, etc.~\cite{bengio2017deep}. However, a complete theory of deep learning is still missing. In particular, until recently, it was not understood why deep networks perform better than other nonparametric methods in classification or regression problems. In a groundbreaking work,~\cite{schmidt2020nonparametric} defined a set of functions that are composition of dimension-reducing smooth functions. He derived the minimax rate of estimation for this class of functions. A very important fact is that this rate can be reached by deep ReLU networks, but not by more classical methods such as wavelets. Since then, these results were extended in many directions. The smoothness of the function was defined in terms of H\"older spaces, this was extended to Besov spaces by~\cite{suzuki2019}. 
{\color{black} For least squares regression, \cite{kohler2021rate} have considered fully connected feedforward neural networks and obtained a similar rate as in \cite{schmidt2020nonparametric}, whereas \cite{jiao2023deep} established a minimax optimal rate that depends polynomially on the input dimension.
For deep learning from  non-i.i.d. data, we refer to \cite{ma2022theoretical, kohler2023rate, feng2023over, kurisu2025adaptive}  for some recent advances on nonparametric regression and for 
\cite{kengne2023penalized,kengne2024deep,
kengne2024sparse, kengne2025excess},
for some recent results on classification and regression, and where the optimality of the bounds established is unclear.
While the estimator used in~\cite{schmidt2020nonparametric} is frequentist, similar results were proven for Bayesian deep neural networks by~\cite{polson2018posterior,sun2022learning,
castillo2024bayesian,castillo2024posterior,
castillo2024deep}, from independent observations.
See also \cite{che2020b, sun2022consistent, steffen2022pac, kong2024posterior, mai2024misclassification} for some recent results based on Bayesian deep neural networks, for regression and classification, and where the minimax rate is established in the i.i.d. setting.
In this new contribution, we perform the PAC-Bayes approach for deep learning from Markov chains and deal with a setting that includes classification and nonparametric regression.
}    
%
%

\medskip

More precisely, we extend the results of~\cite{schmidt2020nonparametric} in two directions. First, we relax the independence assumption on the observations. Instead, the observations can be sampled from a Markov chain with a positive pseudo-spectral gap.
%
%
Then, we study a more general class of learning problems, that includes least-square regression but also classification with the logistic loss. We propose a generalized Bayesian estimator for which, we derive an upper bound on the excess risk. In the case of least-square regression, this bound matches (up to a logarithmic factor) the lower bound in~\cite{schmidt2020nonparametric}. We are not aware of similar lower bounds on the excess risk in the logistic regression case, so we derive such a bound, which matches (up to a logarithmic factor), our upper bound.

\medskip

Generalized Bayesian methods are an extension of Bayesian estimation available \textcolor{black}{when no likelihood is available, or when we do not assume that the likelihood is well-specified}. In this case, we replace the negative log-likelihood by a loss function similar to the ones used in machine learning. This approach is discussed in depth by~\cite{knoblauch2022optimization}. It can be motivated by the fact that such estimators are obtained by minimizing upper bounds on the generalization error known as PAC-Bayes bounds~\cite{mca1998}. It was actually observed empirically that the minimization of PAC-Bayes bounds for neural networks leads to good results in practice~\cite{dzi2017,per2020,clerico2021wide,clerico2022pacdet}. It appears that, PAC-Bayes bounds can also be used to derive oracle inequalities and rates of convergence in statistics~\cite{catoni2007pac,hellstrom2023generalization}. In the i.i.d. setting, PAC-Bayes oracle bounds were actually proven to lead to minimax rates for least-square regression in various nonparametric classes of functions,
 including the class of compositions structured H\"older functions studied in \cite{schmidt2020nonparametric}, see, for instance, \cite{che2020b,steffen2022pac,mai2024misclassification}. We refer the reader to~\cite{alquier2024user} for a comprehensive review on PAC-Bayes bounds, oracle inequalities, and applications to neural networks.

\medskip
PAC-Bayes bounds were extended to non-i.i.d. settings under various assumptions: mixing, weak-dependence, Markov property, among others. The PAC-Bayes oracle inequalities in~\cite{alq2012} lead to slow rates of convergence for weakly-dependent time series. It is to be noted that one of the applications covered in the paper are (shallow) neural networks. Faster rates were proven in~\cite{alq2013}, under more restrictive mixing assumptions. More recently,~\cite{banerjee2021pac} derived PAC-Bayes bounds for Markov chains via $\alpha$-mixing coefficients. Here, we prove a new PAC-Bayes oracle inequality for Markov chains. It is based on a version of Bernstein inequality due to~\cite{paulin2015concentration} that turned out to be very useful in machine learning~\cite{garnier2023hold}.
 We then apply it on the class of compositions based on H\"older functions to derive our rates of convergence for the DNN estimators.
  Note that, the generalized-Bayes estimators in~\cite{alq2012,alq2013,banerjee2021pac} require some prior knowledge on the strength of the dependence between the observations to be calibrated optimally. Our results are no different, in the sense that we need to know a lower bound \textit{a priori} on the pseudo-spectral gap of the observations to calibrate our estimator properly. However, thanks to some recent advances on the estimation of the mixing time of Markov chains~\cite{levin2016estimating,hsu2019mixing,wolfer2024improved}, which itself provides bounds on the pseudo-spectral gap, we can actually get rid of this assumption in some specific cases.

\medskip
The paper is organized as follows. In a first time, we describe the setting of our general learning problem, and the main assumptions: Section~\ref{sec:setting_notations}.
The generalized Bayesian estimator is described in Section~\ref{sec:prior_and_posterior}. 
Section~\ref{sec:oracle} established our PAC-Bayesian oracle inequalities as well as the excess risk of the Bayesian DNN estimator on the class of H\"older functions and the compositions structured class.
Applications to nonparametric regression and classification are carried out in Section~\ref{sec:examples}. All the proofs are gathered in Section~\ref{sec:proofs}.

\section{Setting and notations}
\label{sec:setting_notations}

Let us first introduce a few notations and conventions used in the paper. We let $\N$ denote the set of natural integers. Give a set $A$, we let $|A|$ denote its cardinality. For all $x \in \R$, $\lceil x \rceil$ denotes the  smallest integer $\geq x$ and $\lfloor x \rfloor$ denotes the largest integer $\leq x$.  For any vector $\boldsymbol{x}=(x_1,\cdots,x_d)^T \in \R^d$, $\|\boldsymbol{x}\| = \underset{1\leq i \leq d}{\max} |x_i|$,  $|\boldsymbol{x}|_0=\sum_{i=1}^d \ind(x_i \neq 0) $. We let $M_{p,q}(\R)$ denote the set of matrices of dimension $p\times q$ with coefficients in $\R$.
 For any  $\boldsymbol{x}=(x_{i,j}) \in M_{p,q}(\R)$, we put $ \|\boldsymbol{x}\| = \underset{1\leq i \leq p}{\max} \sum_{j=1}^{q} |x_{i,j}| $.
 For all $\boldsymbol{x}=(x_1,\cdots,x_d)^T \in \R^d$ and $\boldsymbol{\beta}=(\beta_1,\cdots,\beta_d) \in \N_0^d$, $\boldsymbol{x}^{\boldsymbol{\beta}}=x_1^{\beta_1}\cdot \ldots \cdot x_d^{\beta_d}$.
For $(a_n)$ and $(b_n)$ two sequences of real numbers, we write $a_n  \lesssim  b_n$ if there exists a constant $C >0$ such that $a_n \leq b_n$ for all $n \in \N$.
Also, we set $a_n \asymp b_n$ if $a_n \lesssim b_n$ and $b_n \lesssim a_n$.

\medskip
Given  $d \in \N$, and $E_1, E_2$ two subsets of separable Banach spaces equipped with norms $\| \cdot\|_{E_1}$ and $\| \cdot\|_{E_2}$ respectively, $\mathcal{F}(E_1, E_2)$ denotes the set of measurable functions from $E_1$ to $E_2$. For any function $h: E_1 \rightarrow E_2$ and $U \subseteq E_1$,
\[ \| h\|_\infty = \sup_{\boldsymbol{x} \in E_1} \| h(\boldsymbol{x}) \|_{E_2}, ~ \| h\|_{\infty,U} = \sup_{\boldsymbol{x} \in U} \| h(\boldsymbol{x}) \|_{E_2} \text{ and }\]
 \[\lip_\alpha (h) \coloneqq \underset{\boldsymbol{x}_1, \boldsymbol{x}_2 \in E_1, ~ \boldsymbol{x}_1\neq \boldsymbol{x}_2}{\sup} \dfrac{\|h(\boldsymbol{x}_1) - h(\boldsymbol{x}_2)\|_{E_2}}{\| \boldsymbol{x}_1- \boldsymbol{x}_2 \|^\alpha_{E_1}}
 \text{ for any } \alpha \in [0,1] .\]
 For any $\mathcal{K}>0$ and $\alpha \in [0,1]$, $\Lambda_{\alpha,\mathcal{K}}(E_1,E_2)$ (simply $\Lambda_{\alpha,\mathcal{K}}(E_1)$ when $E_2 \subseteq \R$) denotes the set of functions  $h:E_1^u \rightarrow E_2$ for some $u \in \N$, such that  $\|h\|_\infty < \infty$ and  $\lip_\alpha(h) \leq \mathcal{K}$. When $\alpha=1$, we set  $\lip_1 (h)=\lip(h)$ and $\Lambda_{1}(E_1) =\Lambda_{1,1}(E_1,\R)$.

\subsection{Observations and objectives}

Let us consider the supervised learning framework with training sample $D_n=\{\boldsymbol{Z}_1=(\boldsymbol{X}_1,\boldsymbol{Y}_1),\cdots, \boldsymbol{Z}_n=(\boldsymbol{X}_n,\boldsymbol{Y}_n) \}$.
$D_n$ is a trajectory of a stationary and ergodic process $\{\boldsymbol{Z}_t=(\boldsymbol{X}_t,\boldsymbol{Y}_t), ~ t \in \N \}$, with values in $ \mathcal{Z} = \mathcal{X} \times \mathcal{Y}$, where $\mathcal{X}$ is the input space and $\mathcal{Y}$ the output space. We let $P_n$ denote the probability distribution of $D_n$.
In the sequel, we will deal with the case where $\{\boldsymbol{Z}_t=(\boldsymbol{X}_t,\boldsymbol{Y}_t), ~ t \in \Z \}$ is a Markov process.
Also, assume that $\mathcal{X} \subset \R^{d_x}$ and  $\mathcal{Y} \subset \R^{d_y}$, with $d_x, d_y \in \N$.

\medskip

Let  $\ell : \R^{d_y} \times \mathcal{Y} \rightarrow [0,\infty)$ a loss function. For any $h \in \mathcal{F}(\mathcal{X}, \mathcal{Y})$, define the risk,
\begin{equation}\label{def_risk}
  R(h) = \E_{\boldsymbol{Z}_0}\big[\ell \big(h(\boldsymbol{X}_0),\boldsymbol{Y}_0 \big) \big], ~ \text{ with } \boldsymbol{Z}_0 = (\boldsymbol{X}_0,\boldsymbol{Y}_0).
   \end{equation}
Denote by $h^* \in \mathcal{F}(\mathcal{X}, \mathcal{Y})$, a target predictor (assumed to exist), satisfying,
\begin{equation}\label{cond_best_pred}
  R(h^*) =  \underset{h \in \mathcal{F}(\mathcal{X}, \mathcal{Y})}{\inf} R(h) .
  \end{equation}
Define for all $h \in \mathcal{F}(\mathcal{X}, \mathcal{Y})$, the empirical risk,
\begin{equation}\label{def_emp_risk}
 \widehat{R}_n(h) = \frac{1}{n} \sum_{i=1}^n \ell\big(h(\boldsymbol{X}_i),\boldsymbol{Y}_i \big).
 \end{equation}
The aim is to build from the training sample $D_n$, a learner $\widehat{h}_n$, that achieves the smallest excess risk, given by 
\begin{equation}\label{def_exces_risk}
 \mathcal{E}(\widehat{h}_n) = R(\widehat{h}_n) - R(h^*).
 \end{equation}
 Also, for all $h \in \mathcal{F}(\mathcal{X}, \mathcal{Y})$, we put
 \begin{equation}\label{def_emp_exces_risk}
 \widehat{\mathcal{E}}_n(h) = \widehat{R}_n(h) - \widehat{R}_n(h^*).
 \end{equation}
In the following, we set $\ell(h,\boldsymbol{z}) = \ell\big(h(\boldsymbol{x}),\boldsymbol{y}\big)$ for all $\boldsymbol{z}=(\boldsymbol{x},\boldsymbol{y}) \in  \mathcal{X} \times\mathcal{Y}$ and $h \in \mathcal{F}(\mathcal{X}, \mathcal{Y})$.
\subsection{Deep networks and function spaces}

We focus on deep neural networks (DNN) predictors, with activation function $\sigma : \R \rightarrow \R$. 
Let $L \in \N$, $\boldsymbol{p} = (p_0,p_1,\cdots,p_{L+1}) \in \N^{L+2}$. 
A DNN, with network architecture $(L, \boldsymbol{p})$ is any function of the form,

\begin{equation} \label{DNN_def}
 h_{\boldsymbol{\theta}} : \R^{p_0} \rightarrow \R^{p_{L+1}}, ~ x \mapsto h_{\boldsymbol{\theta}} (x) = \boldsymbol{A}_{L+1} \circ \sigma_L \circ \boldsymbol{A}_L \circ \sigma_{L-1} \circ \boldsymbol{A}_{L-1}\circ \cdots \circ \sigma_1 \circ \boldsymbol{A}_1(\boldsymbol{x}),
\end{equation} 
where for any $\ell = 1,\cdots, L+1$, $\boldsymbol{A}_{\ell}: \R^{p_{\ell-1}} \rightarrow \R^{p_{\ell}}$ is an affine application defined by $\boldsymbol{A}_\ell (\boldsymbol{x}) := \boldsymbol{W}_\ell \boldsymbol{x} + \boldsymbol{b}_\ell$ for given weight matrix $\boldsymbol{W}_\ell\in M_{p_{\ell},p_{\ell-1}}(\R)$, a shift vector $\boldsymbol{b}_\ell \in \R^{p_\ell}$ and $\sigma_\ell : \R^{p_\ell} \rightarrow \R^{p_\ell}$ is an element-wise activation map defined as $\sigma_{\ell}(\boldsymbol{z})=(\sigma(z_1),\cdots,\sigma(z_{p_\ell}) )^T $ for all $\boldsymbol{z}=(z_1,\cdots,z_{p_\ell})^T \in \R^{p_\ell}$, and
\begin{equation}\label{def_theta}
\boldsymbol{\theta} = \big( \text{vec}(\boldsymbol{W}_1)^T, \boldsymbol{b}_1^T,\cdots, \text{vec}(\boldsymbol{W}_{L+1})^T, \boldsymbol{b}_{L+1}^T \big),
\end{equation}
is the network parameters (weights), where $\text{vec}(\boldsymbol{W})$ is obtained by concatenating the column vectors of the matrix $\boldsymbol{W}$ and $^T$ denotes the transpose. 
In our setting here, $p_0 = d_x$ (input dimension) and $p_{L+1} = d_y$ (output dimension).
So, the total number of network parameters is
\begin{equation}\label{def_n_L_p}
 n_{L,\boldsymbol{p}} = \sum_{\ell=1}^{L+1} p_{\ell} \times p_{\ell -1} + \sum_{\ell=1}^{L+1} p_{\ell} .
\end{equation}

Let $\mathcal{H}_{\sigma,d_x,d_y}$ denote the set of DNNs of the form (\ref{DNN_def}) with $d_x$ dimensional input, $d_y$ dimensional output, with activation function $\sigma$.
For any $h \in \mathcal{H}_{\sigma,d_x,d_y}$ with a neural network architecture $(L, \boldsymbol{p})$, denote $\text{depth}(h) = L$ and $\text{width}(h) = \underset{1\leq \ell \leq L}{\max} p_\ell$, respectively the depth and  width of $h$. 
For any $L, N, S, B, F>0$, consider the sets,
{\color{black}
\begin{equation}\label{def_H_lnbf}
\mathcal{H}_{\sigma,\dx,\dy} (L,N, B, F) = \big\{ h_{\boldsymbol{\theta}} \in \mathcal{H}_{\sigma,d_x,d_y}, ~ \text{depth}(h) \leq L, ~\text{width}(h) \leq N, ~  ~\|\boldsymbol{\theta}\| \leq B,  ~\|h\|_{\infty,\mx} \leq F \big\},
\end{equation}
\begin{equation}\label{def_H_lnbfs}
\mathcal{H}_{\sigma,\dx,\dy} (L,N, B, F, S) = \big\{ h_{\boldsymbol{\theta}} \in \mathcal{H}_{\sigma,\dx,\dy} (L,N, B, F) ,  ~  |\boldsymbol{\theta}|_0 \leq S \big\}.
\end{equation}
}

\medskip

Let $L, N, B, F >0$. For any $h \in \mathcal{H}_{\sigma,\dx,\dy} (L,N, B, F)$, the maximum number of parameter of $h$ is,
\begin{equation}\label{def_n_L_N}
n_{L,N} :=  N(N+1)(L+1) . 
\end{equation} 
For any $h \in \mathcal{H}_{\sigma,\dx,\dy} (L,N, B, F)$, we will consider in the sequel, $\boldsymbol{\theta}(h)$ as a $n_{L,N}$-dimensional vector by adding $0$ if not.
 For any active set $\mathcal{I} \subseteq \{1,2,\cdots, n_{L,N}\}$, define the corresponding class of sparse DNNs by,
{\color{black}
 \begin{equation}\label{def_H_lnbfI}
\mathcal{H}_{\sigma,\dx,\dy} (L,N, B, F, \mathcal{I}) = \big\{ h_{\boldsymbol{\theta}} \in \mathcal{H}_{\sigma,\dx,\dy} (L,N, B, F ), ~  ~ \theta_i = 0 ~ \text{if} ~ i \notin \mathcal{I}  \big\},
\end{equation}
}
where $\boldsymbol{\theta} := \big( \theta_1,\cdots, \theta_{n_{L,N}}\big)^T$. 
In the sequel, for a DNN $h_{\boldsymbol{\theta}}$, denote,
 \begin{equation}\label{notation_risk_eces_risk}
R(\boldsymbol{\theta}) = R(h_{\boldsymbol{\theta}}), ~  \widehat{R}_n(\boldsymbol{\theta}) = \widehat{R}_n(h_{\boldsymbol{\theta}}) , ~ \mathcal{E}(\boldsymbol{\theta}) = R(\boldsymbol{\theta}) - R(h^*) \text{ and }  \widehat{\mathcal{E}}_n(\boldsymbol{\theta}) = \widehat{R}_n(\boldsymbol{\theta}) - \widehat{R}_n(h^*).
\end{equation} 

\subsection{Markov chains and pseudo-spectral gap}
We assume that $Z=\{\boldsymbol{Z}_t=(\boldsymbol{X}_t,\boldsymbol{Y}_t), ~ t \in \N \}$ is a stationary and ergodic Markov process with transition kernel $P$ and stationary distribution $\pi$, and let $D_n=\{\boldsymbol{Z}_1=(\boldsymbol{X}_1,\boldsymbol{Y}_1),\cdots, \boldsymbol{Z}_n=(\boldsymbol{X}_n,\boldsymbol{Y}_n) \}$ be a trajectory of $Z$.
$P$ can be then viewed as a linear operator $ \bm{P}$ on $L^2(\pi)$ defined as $(\bm{P} f)(z):= \E_{P(z,\cdot)}(f)$ for all $f \in L^2(\pi)$.

{\color{black}
A difficulty in learning with dependent observations is that the dependence can slow down the law of large numbers, and thus the convergence of statistical estimation. It is thus necessary to introduce assumptions on the strength of the dependence. Many notions were introduced to quantify such a notion, such as mixing and weak dependence. We refer the reader to~\cite{rio2017asymptotic} and~\cite{dedecker2007weak} for respective introductions. In the case of Markov chains, this can be characterized by spectral properties of $ \bm{P}$.
}

For any such operator $ \bm{F}$ on $L^2(\pi)$, we can define
\begin{equation}\label{def_S2}
 S_2(\bm{F}) := \{ \xi \in \C : (\xi I - \bm{F}) \text{ is not invertible on }  L^2(\pi)\},
\end{equation}
where $I$ denotes the identity operator on $L^2(\pi)$.
The spectral gap of $\bm{F}$ is given by,
\begin{equation}\label{spec_gap}
\gamma(\bm{F}):=\left\{
\begin{array}{ll}
1- \sup\{\xi \in S_2(\bm{F}): \xi \neq 1 \} & \text{ if eigenvalue 1 has multiplicity 1 } \\
0 & \text{ otherwise}.
\end{array}
\right.
\end{equation}
We will refer to $\gamma(\bm{P})$ simply as the spectral gap of the chain $Z$.
Now, define the time reversal of the kernel $P$ by,
\begin{equation}\label{def_K_star}
 P^*(z,du) := \frac{P(u,dz)}{ \pi(dz)} \pi(du).
\end{equation}
Therefore, the linear operator $\bm{P}^*$ is the adjoint of the linear operator $\bm{P}$ on $L^2(\pi)$.
{\color{black}
Define $\gamma_{ps,n}$, called the pseudo-spectral gap of $\bm{P}$ see \cite{paulin2015concentration}, as
\begin{equation}\label{def_gamma_ps}
 \gamma_{ps,n}  = \underset{k \geq 1}{ \max} \Big\{ \frac{1}{k} \gamma \big( (\bm{P}^*)^k \bm{P}^k \big) \Big\}.
\end{equation}
Note that in order to highlight the role of the spectral gap in the asymptotic regime, we will allow  $\gamma_{ps,n}$ to actually depend on the sample size $n$ (thus $\bm{P}$ can also depend on $n$).
}
{\color{black} Our analysis will essentially show that we can recover rates of convergence for deep learning that are similar to the ones in the i.i.d. case, simply replacing the sample size $n$ by an ``effective sample size'' $\tilde{n} = n\gamma_{ps,n}$. Thus, consistency will be ensured as long a $n \gamma_{ps,n} \rightarrow \infty$ when $n\rightarrow \infty$. Moreover, if $\gamma_{ps,n}\geq \underline{\gamma}$ for some $\underline{\gamma} >0$, we will recover exactly the same rates as in the i.i.d. case. Such a condition on $\gamma_{ps,n}$ is weaker than a similar assumption on $\gamma(\bm{P})>0$, as discussed in Remark 3.2 of~\cite{paulin2015concentration}. Let us now briefly discuss examples where this assumption is satisfied. In order to do so, we introduce the more classical notion of mixing time.
}

\medskip

 For the process $Z=\{\boldsymbol{Z}_t=(\boldsymbol{X}_t,\boldsymbol{Y}_t), ~ t \in \N \}$, the mixing time for some $\epsilon >0$ is defined as,
 \begin{equation}\label{def_mixing_time}
 \tmix(\epsilon):= \min \big\{ t \in \N, ~ \underset{z \in \mz}{\sup}\| P^t(\boldsymbol{z},\cdot) - \pi\|_{TV} \leq \epsilon \big\}, \text{ and } \tmix := \tmix(1/4),
\end{equation}
where $P^t(\boldsymbol{z},\cdot)$ is the  $t$-step transition kernel (i.e. the distribution of $\boldsymbol{Z}_t|\boldsymbol{Z}_0=\boldsymbol{z}$), and $\|\cdot \|_{TV} $ denotes the total variation distance between two probability measures $Q_1, Q_2$ defined on the same state space $(\Omega, \mathcal{F})$ and is given by,
 \begin{equation}\label{def_norm_TV}
 \|Q_1 - Q_2 \|_{TV}:= \underset{A \in \mf}{\sup} |Q_1(A) - Q_2(A)|.
\end{equation}
{\color{black}
Intuitively, the mixing time is indeed related to the dependence strength, as it quantifies how long we have to wait for the chain to ``loose memory'' of its initial point $z$ and reach the stationary regime. Actually,
\cite{paulin2015concentration} derived relations between mixing time and pseudo-spectral gap (Propositions 3.3 and 3.4), in the setting of the following example.
}
\begin{example}
Assume that the chain $(\boldsymbol{Z}_t)_{t \in \N}$ is uniformly ergodic, that is, there are constants $C>0$ and $\rho<1$ such that for all $t \in \N$,
$$ \underset{z \in \mz}{\sup}\| P^t(\boldsymbol{z},\cdot) - \pi\|_{TV} \leq C \rho^t . $$
In this case, Proposition 3.4 in \cite{paulin2015concentration} shows that $ \gamma_{ps,n} \geq 1/(2 \tmix)$.
Thus, we have $\tmix(\epsilon) \leq \lceil \log(C/\epsilon)/\log(\rho)\rceil $ and in particular, $ \gamma_{n} \geq 1/(2 \tmix) = \lceil \log(4C)/\log(\rho)\rceil >0 $.

\noindent
When the set $\mathcal{Z}$ is finite, any irreducible and aperiodic Markov chain is immediately uniformly ergodic~\cite{DoucMARKOV}. Examples of uniformly ergodic chains on infinite state-space are also well known~\cite{DoucMARKOV}.
\end{example}

{\color{black}
\begin{example}
A standard framework for nonparametric regression is to assume that the $(\boldsymbol{X}_t)_{t \in \N}$ are i.i.d. with uniform distribution $\pi$ on $[0,1]^d$, and $\boldsymbol{Y}_t:=f ^*(\boldsymbol{X}_t)+\varepsilon_t$ for some measurable function $f^*$ and an i.i.d. sequence $(\varepsilon_t)_{t \in \N}$, where $\varepsilon_t$ is independent of the input $X_t$. 
We can generalize this setting. Let us keep the relation $\boldsymbol{Y}_t:=f ^*(\boldsymbol{X}_t)+\varepsilon_t$ but, assume we have two sources of the input $\boldsymbol{X}_t$ on $[0,1]^d$: the first source produces examples uniformly distributed on $A$, and the second source produces examples uniformly distributed on $B$, with $\pi(A)=\pi(B)=1/2$, $A\cap B=\emptyset$ and $A\cup B=[0,1]^d$. The source we sample from is given by a Markov chain $(\boldsymbol{U}_n)_{n\in\mathbb{N}}$, taking values in $\{A,B\}$, with transition $P_U(A,A)=P_U(B,B)=1-p$ and $P_U(A,B)=P_U(B,A)=p$, for some $p \in [0,1]$.
In other words, $(\boldsymbol{X}_t)_{t \in \N}$ is a Markov chain with transition kernel given by,
\[
P_X(x,{\rm d}x') = 2 \left[ p (\mathbf{1}_{x\in A,x'\in B}+\mathbf{1}_{x\in B,x'\in A})  + (1-p)(\mathbf{1}_{x\in A,x'\in A}+\mathbf{1}_{x\in B,x'\in B}) \right]{\rm d}x'.
\]
Direct calculations show that $P_X$ is a rank 2 operator with eigenvalues $\{1,1-2p,0\}$, and if $p\leq 1/2$, then its pseudo-spectral gap is equal to $1-(1-2p)^2=4p-4p^2=4p(1-p)$.
Thus, the pseudo-spectral gap $\gamma_{ps,n}$ of the chain $(\boldsymbol{Z}_t)_{t\in\mathbb{N}}$ also satisfies $\gamma_n=4p(1-p)$. As will be proven below, the convergence (of the excess risk of the DNN predictor) in this setting will depend on whether the ``effective sample size'' $\tilde{n}=n\gamma_{ps,n}$ satisfies $\tilde{n}\rightarrow\infty$. In the case $p=1/2$, in which the $\boldsymbol{X}_t$ are actually i.i.d. on $[0,1]$, we have $\gamma_{ps,n}=1$ and $\tilde{n}=n$. However, when $p<1/2$, $\tilde{n}=4np(1-p)$. In a sequence of experiments where $p=p_n\rightarrow 0$ when $n\rightarrow\infty$, $\tilde{n}=4np_n(1-p_n) \sim 4np_n$. 
\end{example}
}

\section{Prior distribution and Gibbs estimator}
\label{sec:prior_and_posterior}
Let $L, N, B, F >0$. For a given active set  $\mathcal{I} \subseteq \{1,2,\cdots, n_{L,N}\}$, the prior $\Pi_{\mathcal{I}}$ on the networks parameter of the DNNs class $\mathcal{H}_{\sigma,\dx,\dy} (L,N,B, F, \mathcal{I})$ is set as the uniform distribution on
 \begin{equation}\label{def_S_I}
 \mathcal{S}_{\mathcal{I}} = \big\{ \boldsymbol{\theta}=\big( \theta_1,\cdots, \theta_{n_{L,N}} \big)^T \in [-B, B]^{n_{L,N}}, ~  \theta_i = 0 ~\text{if} ~ i \notin \mathcal{I} \big\}.
\end{equation}
Let us define the mixture of the uniform priors on $\mathcal{S}_{\mathcal{I}}$ as,
 \begin{equation}\label{def_Pi}
\Pi = \sum_{i=1}^{n_{L,N}} s^{-i} \underset{|\mathcal{I}|=i} { \sum_{ \mathcal{I} \subseteq \{1,2,\cdots, n_{L,N}\}, } } \binom{ n_{L,N} }{ i }^{-1}   \Pi_{ \mathcal{I} } \big/  C_s     ~ ~ \text{ with } ~ ~ C_s:= (1-s^{- n_{L,N}})/ (s-1),
\end{equation}
where $s \geq 2$ is a sparsity parameter.
{\color{black}
The prior in~\eqref{def_Pi} is referred to as a spike-and-slab prior in the literature. It is popular in high-dimensional Bayesian statistics as a sparsity inducing prior. Up to our knowledge, they were introduced in high-dimensional linear regression~\cite{mitchell1988bayesian,ishwaran2005spike}. We refer the reader to~\cite{banerjee2021bayesian,nickl2022foundations} for recent reviews on the topic. In deep learning, such priors were used by~\cite{polson2018posterior,che2020b,sun2022learning,castillo2024adaptation} among others. In PAC-Bayes bounds, it was used by \cite{alquier2011pac,steffen2022pac,mai2024misclassification}. Alternative priors for deep networks are discussed in~\cite{fortuin2022priors}.
}
 The Gibbs posterior probability distribution $\Pi_\lambda (\cdot | D_n) $ based on $\Pi$ is defined via,
\begin{equation}\label{def_Pi_lambda}
 \Pi_\lambda (d \boldsymbol{\theta} | D_n) \propto \exp \big( - \lambda \widehat{R}_n(\boldsymbol{\theta}) \big) \Pi(d\boldsymbol{\theta}), 
\end{equation} 
where $\lambda > 0$ is the so-called temperature parameter and the empirical risk defined in (\ref{def_emp_risk}). 
For any $\lambda > 0$, the corresponding weight $\widehat{\boldsymbol{\theta}}_\lambda$ and predictor $\widehat{h}_{\lambda}$ are draw from the posterior distribution, that is,
\begin{equation}\label{theta_lambda_h_lambda}
 \widehat{\boldsymbol{\theta}}_\lambda \sim \Pi_\lambda (\cdot | D_n), ~  \text{ and } ~ \widehat{h}_{\lambda} := h_{\widehat{\boldsymbol{\theta}}_\lambda}.
\end{equation} 
%
%
For $L,N, B, F, S >0$, if we consider the DNN class $\mathcal{H}_{\sigma,\dx,\dy} (L,N,B,F)$ or $\mathcal{H}_{\sigma,\dx,\dy} (L,N,B, F, S)$, $\widehat{\boldsymbol{\theta}}_\lambda$ is defined as above, where for any $\mathcal{I} \subseteq \{1,2,\cdots, n_{L,N}\}$, the prior $\Pi_{\mathcal{I}}$ on the networks parameter of $\mathcal{H}_{\sigma,\dx,\dy} (L,N,B, F, \mathcal{I})$ is set as the uniform distribution on $\mathcal{S}_{\mathcal{I}}$ (given in (\ref{def_S_I})). 
{\color{black} Sampling from posteriors on deep neural networks is not an easy task; still, it was successfully implemented by some authors~\cite{zhang2019cyclical}, including~\cite{steffen2022pac} who sampled from~\eqref{theta_lambda_h_lambda}. Approximation algorithms are more popular in practice, including Laplace approximations~\cite{daxberger2021laplace,cinquin2024fsp} or variational inference~\cite{osawa2019practical}, with theoretical guarantees~\cite{che2020b}.
}

{\color{black} The temperature parameter $\lambda>0$ acts on the balance between data fit and regularization. In the limit case $\lambda=0$, the posterior is equal to the prior, and there is no data fit. On the other hand, when $\lambda\rightarrow\infty$, the regularizing effect of the prior vanishes and the posterior concentrates on the empirical risk minimizer. A thorough discussion on this parameter can be found in Sections 2.1.4 and 4 in~\cite{alquier2024user}. While we provide below a formula to calibrate $\lambda$ in order to obtain minimax rates, this value might not be the best one in practice. In some applications, it seems that $\lambda$ can be taken much larger than what is prescribed by the PAC-Bayes theory. This phenomenon, known as the \textit{cold posterior} effect, is the object of intensive debate and investigation~\cite{wenzel2020good,noci2021disentangling,nabarro2022data}.
}

\section{Oracle inequality and excess risk bounds }
\label{sec:oracle}
\subsection{Oracle inequality}
Let $Z=(\boldsymbol{Z}_t)_{t \in \N}$ (with $\boldsymbol{Z}_t=(\boldsymbol{X}_t,\boldsymbol{Y}_t)$) be a process with values in $\mz = \mx \times \my \subset \R^{\dx} \times \R^{\dy}$.
 Consider the learning problem of $Y_t$ from $\boldsymbol{X}_t$ based on the observations $D_n=\{\boldsymbol{Z}_1=(\boldsymbol{X}_1,\boldsymbol{Y}_1),\cdots, \boldsymbol{Z}_n=(\boldsymbol{X}_n,\boldsymbol{Y}_n) \}$. 
 We carry out the PAC-Bayes approach on the class of DNNs predictor in (\ref{def_H_lnbfs}) and set the following assumptions on the process $Z$, the input space, the loss and the activation. 
%
%
 \begin{enumerate}
  \item[{\color{black}(\textbf{A1}):}]  {\color{black}The  training sample $D_n$ is a trajectory of a stationary and ergodic Markov process $Z$ with transition kernel $K$, stationary distribution $\pi$ and  pseudo spectral gap $\gamma_{ps,n}>0$ that satisfies $n\gamma_{ps,n}\rightarrow\infty$ as $n\rightarrow \infty$.
   In the sequel, we set,
   \begin{equation}\label{def_gamma_n}
   \gamma_{n} := \gamma_{ps,n} .
\end{equation}    
 }
  \item [(\textbf{A2}):] $\mathcal{X} \subset \R^{\dx}$ is a compact set.\\
  Under (\textbf{A2}), we set,
  \begin{equation} \label{def_K_X} 
   \mk_{\mx} = \sup_{\boldsymbol{x} \in \mx} \|\boldsymbol{x}\| .
\end{equation}
 \item [(\textbf{A3}):]  There exists $C_{\ell} >0$, such that, 
 for the DNN class $\mathcal{H}_{\sigma,\dx,\dy} (L,N,B,F)$, for $L,N,B,F>0$, we have for all $h \in \mathcal{H}_{\sigma,\dx,\dy} (L,N,B,F)$ and $(\boldsymbol{x},\boldsymbol{y})\in\mathcal{Z}$, $| \ell(h(\boldsymbol{x}),\boldsymbol{y})-\ell(h^*(\boldsymbol{x}),\boldsymbol{y})| \leq \mathcal{K}_\ell |h(\boldsymbol{x})-h^*(\boldsymbol{x})|$. 
  \item [(\textbf{A4}):] There exists $C_{\sigma} >0$, such that $\sigma \in \Lambda_{1,C_{\sigma}}(\R) $. Moreover, $\sigma$ is either piecewise linear or locally quadratic and fixes a segment $I \subseteq [0,1]$. 
\item [(\textbf{A5}):] There exists $\mk >0$, such that the target function defined in (\ref{cond_best_pred}), {\color{black}$h^*: \mx \rightarrow \R^\dy$, satisfies $\lip(h^*) \leq \mk$}.
%
 \end{enumerate}
Let us recall that, a function $g: \R \rightarrow \R$ is locally quadratic (see also \cite{ohn2019smooth}, \cite{ohn2022nonconvex}) if there exists an interval $(a,b)$ on which $g$ is three times continuously differentiable with bounded derivatives and there exists $t \in (a,b)$ such that $g'(t) \neq 0$ and $g''(t) \neq 0$.
(\textbf{A4}) is satisfied by the ReLU activation and several other activation functions; see, for instance \cite{kengne2025excess}. 

\medskip

%
%
Let us set the Bernstein's condition:
\begin{itemize}
\item [(\textbf{A6}):] The Bernstein's condition for some $K>0$ and $\kappa \in [0,1]$ : For all $\theta \in \mathcal{S}_{\mathcal{I}}$,
  \begin{equation}\label{eq_assup_Bernstein}
   \E_\pi\Big[ \Big(   \ell\big(h_{\boldsymbol{\theta}}(\boldsymbol{X_i}),\boldsymbol{Y_i} \big)  - \ell\big(h^*(\boldsymbol{X_i}),\boldsymbol{Y_i} \big)  \Big)^2 \Big]  \leq  K \big( R(\boldsymbol{\theta}) - R(h^*) \big)^\kappa  .
 \end{equation}
\item [(\textbf{A7}):] Local quadratic structure of the excess risk: There exist two constants $\mk_0:= \mk_0(\boldsymbol{Z}_0, \ell, h^*) , \varepsilon_0:= \varepsilon_0(\boldsymbol{Z}_0, \ell, h^*) >0$ such that,  
\begin{equation}\label{assump_local_quadr}
 R(h) - R(h^*) \leq \mk_0 \| h - h^*\|^2_{2, P_{\boldsymbol{X}_0}}, 
\end{equation}
for any measurable function $h: \R^\dx \rightarrow \R^\dy$ satisfying $\|h - h^*\|_{\infty, \mx} \leq \varepsilon_0$; where $P_{\boldsymbol{X}_0}$ denotes the distribution of $\boldsymbol{X}_0$ and 
\[ \| h - h^{*}\|_{q, P_{\boldsymbol{X}_0}}^q  := \displaystyle \int \| h (\text{x})- h^{*} (\text{x}) \|^q  d P_{\boldsymbol{X}_0} ( \text{x}),  
\]
for all $q \geq 1$.

\end{itemize}
%
%
%
%
 For all $L, N, B, F > 0$ and $\mathcal{I} \subseteq \{1,2,\cdots, n_{L,N}\}$, define,
 \begin{equation}\label{theta_star_I}
  \boldsymbol{\theta}^*_{\mathcal{I}} \in \argmin_{\theta \in \mathcal{S}_{\mathcal{I}}} R(\boldsymbol{\theta}),
\end{equation} 
where $\mathcal{S}_{\mathcal{I}}$ is defined in (\ref{def_S_I}).
  $\boldsymbol{\theta}^*_{\mathcal{I}}$ in (\ref{theta_star_I}) is well defined, since $\mathcal{S}_{\mathcal{I}}$ is compact and the function $\theta \mapsto R(\boldsymbol{\theta})$ is continuous.
The following theorem provides a non-asymptotic oracle inequality of the excess risk. It is essentially an extension of Corollary 3 from~\cite{steffen2022pac} to the case of Markov chains. {\color{black} The proof relies on PAC-Bayes techniques developed in~\cite{catoni2007pac} and summarized in~\cite{alquier2024user}. PAC-Bayes bounds require a deviation inequality to start with. As the results in~\cite{catoni2007pac} are in the i.i.d. or the exchangeable case, we use here Paulin's version of Bernstein's inequality (Theorem 3.4 in~\cite{paulin2015concentration}, see also Lemma \ref{lem_Paulin} below). We then adapt the proof of Theorem 4.3 in~\cite{alquier2024user}.}
\begin{thm}\label{theo_oracle_inq}
 Assume that (\textbf{A1})-(\textbf{A5}) and (\textbf{A6}) with $\kappa=1$ hold and take
 \begin{equation}\label{theo_oracle_lambda}
 \lambda =\dfrac{n \gamma_{n}}{ 32K + 10},
 \end{equation}
 where $K$ is given in (\ref{eq_assup_Bernstein}).
Consider the DNNs class $\mathcal{H}_{\sigma,\dx,\dy} (L,N, B, F)$ for $L,N,B, F>0$, and $\widehat{\boldsymbol{\theta}}_\lambda$ defined as in (\ref{theta_lambda_h_lambda}).
Then, for all $\delta \in (0,1)$, with $P_n \otimes \Pi_\lambda$-probability at least $1-\delta$, we have
 \begin{equation*} 
   \mathcal{E}(\widehat{\boldsymbol{\theta}}_\lambda) \leq  \min_{ \mathcal{I} \subseteq \{1,2,\cdots, n_{L,N}\} } \Bigg(  3 \mathcal{E}( \boldsymbol{\theta}^*_{\mathcal{I}} ) +  \dfrac{ \Xi_1 }{n \gamma_{n}} \Big[|\mathcal{I}| L \log\big( \max(n, B, n_{L,N} )  \big)   +  \log(1/\delta) + C_{\ell} \Big]    \Bigg)
 \end{equation*}
for some constant $\Xi_1$ independent of $n, L, N, B, F, \gamma_{n}$ and $ C_{\ell} $.
\end{thm}

{\color{black}
\begin{rmrk}
To compute the estimator, which depends on $\lambda$, we need to know the spectral gap $\gamma_{n}$, or at least an upper bound on it. In some situations, such upper bounds might be available \textit{a priori}, but this is not always the case. The estimation of $\gamma_{n}$ from a single trajectory of the chain was studied in recent works~\cite{levin2016estimating,hsu2019mixing,wolfer2024improved}. In the case of a finite state-space, it is possible to estimate $\gamma_{n}$ with a confidence interval: from Theorem 2.1 in~\cite{wolfer2024improved}, there is an estimator $\hat{\gamma}_n$ such that $\hat{\gamma}_n\in[\gamma_n/2,3\gamma_n/2]$ with probability at least $1-\delta$ as soon as $
n\geq \frac{4c}{\gamma_{n}^3\pi_*} \log\frac{1}{\pi_*} \log\frac{2}{\pi_* \gamma_{n}\delta}   \log \frac{2}{\pi_* \gamma_{n}^2 \delta}
$
(simply take $\varepsilon=\gamma_n/2$ in their theorem).
This means we can take $\lambda = n \hat{\gamma}_{n}/(32K+10)$ and ensure
 \begin{equation*}
   \mathcal{E}(\widehat{\boldsymbol{\theta}}_\lambda) \leq  \min_{ \mathcal{I} \subseteq \{1,2,\cdots, n_{L,N}\} } \Bigg(  3 \mathcal{E}( \boldsymbol{\theta}^*_{\mathcal{I}} ) +  \dfrac{ 2 \Xi_1 }{n \gamma_{n}} \Big[|\mathcal{I}| L \log\big( \max(n, B, n_{L,N} )  \big)   +  \log(1/\delta) + 3 C_{\ell} \Big]    \Bigg).
 \end{equation*}
The situation is more complicated when the state-space is infinite, as discussions in~\cite{wolfer2024improved} tend to show the estimation of $\gamma_{n}$ is not possible in general. However, it might remain feasible under additional assumptions on the Markov chain. This will be the object of future works.
\end{rmrk}
}

\subsection{Excess risk bound on the H\"older class}

 In the sequel, we consider a class of H\"older smooth functions.
Let $ U \subset \R^d$ (with $d \in \N$), $\boldsymbol{\alpha} = (\alpha_1, \cdots, \alpha_d)^T \in \N^d, \quad \boldsymbol{x} = (\boldsymbol{X}_1, \cdots, x_d)^T \in U$, set
\[ |\boldsymbol{\alpha}|= \sum_{i =1}^{d} \alpha_i ~ \text{and} ~ \partial^{\boldsymbol{\alpha}} = \dfrac{\partial^{|\boldsymbol{\alpha}|}}{\partial^{\alpha_1} \boldsymbol{X}_1, \cdots, \partial^{\alpha_d} x_d}.  \]
For any $ \beta > 0$, the H\"older space $ \mathcal{C}^\beta (U)$ is the set of functions $ h: U \rightarrow \R^\dy$ such that, for any $ \boldsymbol{\alpha} \in \N^d$ with $ |\boldsymbol{\alpha}| \leq \lfloor \beta \rfloor, ~ \| \partial^{\boldsymbol{\alpha}} h \|_{\infty} < \infty$ and for any $ \boldsymbol{\alpha} \in \N^d$ with $ |\boldsymbol{\alpha}| = \lfloor \beta \rfloor, ~ \lip_{s - [s]} (\partial^{\boldsymbol{\alpha}} h) < \infty$. This space is equipped with the norm
\begin{equation}\label{def_Holder_norm}
\| h\|_{\mathcal{C}^{\beta} (U)} =  \underset{ \boldsymbol{\alpha} \in \N^d:  ~|\boldsymbol{\alpha}| < \lfloor \beta \rfloor}{\sum} \| \partial^{\boldsymbol{\alpha}} h \|_{\infty} +  \underset{\boldsymbol{\alpha} \in \N^d: ~|\boldsymbol{\alpha}| = \lfloor \beta \rfloor}{\sum} \lip_{s - \lfloor \beta \rfloor} (\partial^{\boldsymbol{\alpha}} h).  
\end{equation} 
 For any $ \beta > 0, ~  U \subset \R^d $ and $ \mathcal{K} > 0$, consider the set, 
 \begin{equation}\label{def_C_s_K_U} 
 \mathcal{C}^{\beta, \mathcal{K}}(U) = \{h \in  \mathcal{C}^{\beta} (U), ~ \|h\|_{ \mathcal{C}^{\beta} (U)} \leq \mathcal{K}\}. 
 \end{equation}
%

\medskip

\medskip

Now, we deal with the network architecture parameters $L_n, N_n, B_n, F_n$ that depend on the sample size $n$.
An application of Theorem~\ref{theo_oracle_inq} together with a tight control of $\mathcal{E}( \boldsymbol{\theta}^*_{\mathcal{I}} )$ on the H\"older's class leads to the following result.
{\color{black}
\begin{thm}\label{theo_bound_Holder}
 Assume that (\textbf{A1})-(\textbf{A5}), (\textbf{A6}) with $\kappa=1$, (\textbf{A7}) hold and that there exists $\beta, \mk >0$, such that $h^* \in \mathcal{C}^{\beta, \mathcal{K}}(\mx)$.
 Take $\lambda$ as in (\ref{theo_oracle_lambda}).
Consider the DNNs class $\mathcal{H}_{\sigma,\dx,\dy} (L_n,N_n,  B_n, F_n, S_n)$ with $L_n \asymp \log n, N_n \asymp (n\gamma_n)^{\frac{d_x}{2\beta+d_x}}, ~ S_n \asymp  (n\gamma_n)^{\frac{d_x}{2\beta+d_x}} \log (n\gamma_n), ~ B_n \asymp (n\gamma_n)^{\frac{4(\beta+\dx)}{2\beta+\dx}}$ and $F_n >0$.
Then, for all $\delta \in (0,1)$, with $P_n \otimes \Pi_\lambda$-probability at least $1-\delta$, we have for $n\gamma_n \geq \varepsilon_0^{-(2 + \dx/s) }$,
 \begin{equation*}
   \mathcal{E}(\widehat{\boldsymbol{\theta}}_\lambda) \leq
  \Xi_2 \Bigg( \dfrac{ \log^3 ( n\gamma_n ) }{(n\gamma_n)^{\frac{2\beta}{2\beta+d_x}}}     +  \dfrac{   \log(1/\delta) + C_{\ell}}{(n\gamma_n)} \Bigg)
   ,
 \end{equation*}
for some constant $\Xi_2$ independent of $n$ and $C_{\ell}$, and where $\gamma_n$, $\varepsilon_0$ are given in (\ref{def_gamma_n}) and in the assumption (\textbf{A7}) respectively.
\end{thm}
}
\noindent
Therefore, {\color{black} when $\mathbf{P}$ and thus $\gamma_n=\gamma$ are independent of $n$, the rate of the excess risk of $\widehat{h}_{\lambda} = h_{\widehat{\boldsymbol{\theta}}_\lambda}$ is of order $\mathcal{O}\big( n^{- \frac{2\beta}{2\beta+d_x}} (\log n)^3 \big)$. More generally, if there exists $ \underline{\gamma} > 0 $ such that $\gamma_n \geq \underline{\gamma}$ for all $n \in \N$, we also recover this rate. However, when $\gamma_n\rightarrow 0$, we get slower rates of convergence.
}

\subsection{Excess risk bound on a compositions structured H\"older class}
%
%
%
In this subsection, we assume that $d_y=1$ and that the target function $h^*$ belongs to a class of composition structured functions, defined by \cite{schmidt2020nonparametric}.

 \medskip

 Let $d, r \in \N$ with $d \geq r$, $\beta, \mathcal{K}>0$ and $a, b \in \R$ with $a<b$. 
 Denote by $\mathcal{C}^{\beta,\mathcal{K}}_r([a,b]^d)$ the set of functions $f : [a,b]^d \rightarrow \R$ satisfying:
 \begin{itemize}
 \item $f$ is a $r$-variate function i.e. $f$ depends on at most $r$ coordinates;
 \item $f \in \mathcal{C}^{\beta,\mathcal{K}}([a,b]^d)$, defined in (\ref{def_C_s_K_U}).
 \end{itemize}

\medskip

 Let $q \in \N$, $\boldsymbol{d} = (d_0,\cdots,d_{q+1}) \in \N^{q+2}$ with $d_0 = d_x$ and $d_{q+1} = d_y$, $\boldsymbol{t} = (t_0,\cdots,t_{q}) \in \N^{q+1}$ with $t_i \leq d_i$ for all $i=0,\cdots,q$, $\boldsymbol{\beta} = (\beta_0,\cdots,\beta_{q}) \in (0,\infty)^{q+1}$ and $\mathcal{K} >0$.
 Following~\cite{schmidt2020nonparametric}, define the class of functions:
\begin{multline}\label{def_comp_functions}
\mathcal{G}(q, \boldsymbol{d}, \boldsymbol{t}, \boldsymbol{\beta}, \mathcal{K}) := \Big\{ h=g_q \circ \cdots g_0, ~ g_i=(g_{i,j})_{j=1,\cdots,d_{i+1}} : [a_i, b_i]^{d_i} \rightarrow [a_{i+1}, b_{i+1}]^{d_{i+1}}, \text{ with } \\ 
g_{i,j} \in \mathcal{C}^{\beta_i,\mathcal{K}}_{t_i}\big([a_i,b_i]^{d_i} \big), 
  \text{ for some } a_i, b_i \in \R \text{ such that } |a_i|, |b_i| \leq \mathcal{K}, \forall i=1,\cdots, q \Big\}. 
\end{multline}
For the class $\mathcal{G}(q, \boldsymbol{d}, \boldsymbol{t}, \boldsymbol{\beta}, \mathcal{K})$, define:
\begin{equation}\label{def_beta_star_phi_n}
 \beta^*_i :=\beta_i \prod_{k=i+1}^q \min(\beta_k, 1) \text{ for all } i=0,\cdots, q ; \text{ and } \phi_n := \max_{i=0,\cdots,q}  n^{- \frac{2\beta_i^*}{2 \beta_i^* + t_i}}.
\end{equation}
%
%
%
%
The next theorem provides a bound of the excess risk of the estimator $\widehat{h}_{\lambda}$, for structured composite target function, in the case $\mx = [0,1]^\dx$. Here again this is obtained by an application of Theorem~\ref{theo_oracle_inq}.
{\color{black}
\begin{thm}\label{theo_bound_comp}
 Set $\mx = [0,1]^\dx$. 
 Assume that (\textbf{A1})-(\textbf{A3}), (\textbf{A6}) with $\kappa=1$, (\textbf{A7}) hold and that $h^* \in \mathcal{G}(q, \boldsymbol{d}, \boldsymbol{t}, \boldsymbol{\beta}, \mathcal{K})$ for some composition structured class, and take $\lambda$ as in (\ref{theo_oracle_lambda}).
Consider the DNNs class $\mathcal{H}_{\sigma,\dx,\dy} (L_n,N_n,  B_n, F_n, S_n)$ where $\sigma$ is the ReLU activation function and with:
\begin{equation}\label{theo_upper_cond_architecture}
L_n= C_0 \log (n\gamma_n), N_n \asymp n\gamma_n \phi_{n\gamma_n}, ~ S_n \asymp  n \gamma_n \phi_{n\gamma_n}  \log (n\gamma_n), ~ B_n =B \geq 1, F_n > \max(\mathcal{K}, 1),
\end{equation}
for some constant $C_0 \geq  \sum_{i=0}^q \log_4\big(4 \max(t_i, \beta_i) \big)$ and where $\phi_n$ is defined in (\ref{def_beta_star_phi_n}).
Then, for all $\delta \in (0,1)$, with $P_n \otimes \Pi_\lambda$-probability at least $1-\delta$, we have
 \begin{equation*}
   \mathcal{E}(\widehat{\boldsymbol{\theta}}_\lambda) \leq   \Xi_3  \Big( \phi_{n\gamma_n} (\log (n\gamma_n))^3 + \dfrac{\log(1/\delta) + C_{\ell}}{n\gamma_n } \Big),
 \end{equation*}
for some constant $\Xi_3$ independent of $n$ and $ C_{\ell}$.
\end{thm}
}
\noindent
So, if $h^*$ belongs to a class of H\"older compositions functions and $\gamma_n\geq \underline{\gamma}>0$ holds, then the rate of the excess risk of $\widehat{h}_{\lambda} = h_{\widehat{\boldsymbol{\theta}}_\lambda}$ is of order $\mathcal{O}\big( \phi_n (\log n)^3 \big)$.

\section{Examples of nonparametric regression and classification}
\label{sec:examples}
\subsection{Nonparametric regression}
Consider the nonparametric time series regression,
\begin{equation}\label{npar_reg_model}
\boldsymbol{Y}_t = h^ {*}(\boldsymbol{X}_t) + \varepsilon_t, ~ t \in \N,
\end{equation}
with $ (\boldsymbol{X}_t,\boldsymbol{Y}_t) \in \mx \times \my \subset {\color{black}\R^\dx \times  \R^\dy}$,
where $ h^{*} : {\color{black} \R^\dx  \rightarrow \R^\dy}$  is the  unknown regression  function and $(\varepsilon_t)_{t \in \N}$ is a sequence of {\color{black}$\dy$-dimensional} i.i.d. centered random variables and for all $t \in \N$, $\varepsilon_t$ is independent of the input variable $\boldsymbol{X}_t$.
 The goal is the estimation of $h^*$, based on the observations $\{ (\boldsymbol{X}_i,\boldsymbol{Y}_i) \}_{1\leq i \leq n}$, generated from (\ref{npar_reg_model}).
 We perform the quasi-Bayesian DNN estimator proposed above, with a DNNs class $\mathcal{H}_{\sigma,\dx,\dy} (L_n,N_n,  B_n, F_n, S_n)$ and the square loss function.

\medskip
 In this subsection, assume that $\{ (\boldsymbol{X}_t,\boldsymbol{Y}_t) \}_{t \in \N}$ is a stationary and ergodic Markov process and that $\mx$ and $\my$ are compact sets.
 Hence, the assumptions (\textbf{A1})-(\textbf{A3}) hold.
 Also, (\textbf{A6}) with $\kappa=1$ holds if we deal with a network architecture parameters $L_n, N_n, B_n, F_n$ with $F_n=F>\max(\mathcal{K}, 1)$.
 Note that, in this framework with the square loss, (\textbf{A7}) holds automatically.
 Therefore:
 \begin{itemize}
 \item If the activation function $\sigma$ satisfies (\textbf{A4}) and $h^* \in \mathcal{C}^{\beta, \mathcal{K}}(\mx)$ for some $\beta, \mk>0$, then, the result of Theorem \ref{theo_bound_Holder} holds.
 \item For $\mx = [0,1]^\dx$ and $\dy=1$, if $h^* \in \mathcal{G}(q, \boldsymbol{d}, \boldsymbol{t}, \boldsymbol{\beta}, \mathcal{K})$ for some composition structured class, then the result of Theorem \ref{theo_bound_comp} is applied with the ReLU deep neural networks.
\end{itemize}   

\medskip
Note that, for a regression problem with the Gaussian error $\varepsilon_t$ (with $\dy=1$), square loss, and $\mx = [0,1]^\dx$, \cite{schmidt2020nonparametric} proved that, the lower bound for the minimax estimation risk over a class $\mathcal{G}(q, \boldsymbol{d}, \boldsymbol{t}, \boldsymbol{\beta}, \mathcal{K})$ is $\phi_n$ (up to a constant). 
 Unfortunately, Theorem \ref{theo_bound_Holder} and \ref{theo_bound_comp} do not cover the Gaussian error case, as they would require $ \my=\mathbb{R}$ to be not compact, a setting in which the quadratic loss is not Lipschitz. This can, however, be fixed, {\color{black}so that this theorem can be applied to models with sub-Gaussian errors}.

\begin{lem} \label{lem_reg}
 Assume~\eqref{npar_reg_model} with $\dy=1$, $h^* \in \mathcal{C}^{\beta,\mathcal{K}}(\mathcal{X})$ for some $s, \mathcal{K} >0$,  $(\varepsilon_t)_{t \in \N}$ {\color{black}is sequence of centered i.i.d. sub-Gaussian random variables with variance proxy $\varsigma^2>0$ (that is $\E(e^{\alpha \varepsilon_0}) \leq e^{\alpha^2 \varsigma^2/2}$ for all $\alpha \in \R$)}.
 For any $\delta>0$, define the event
 $$ \tilde{\Omega}_{\delta} = \left\{ \sup_{1\leq t \leq n} |Y_t| \leq  \mathcal{K} +  \varsigma \sqrt{ 2\log \frac{2n}{\delta} }  \right\}. $$
 Then $\mathbb{P}(\tilde{\Omega}_\delta)\geq 1-\delta$.
\end{lem}

\begin{proof}
 Fix $\alpha >0$ and $\epsilon>0$. By using the Markov's inequality, we have,
 \begin{align*}
  \mathbb{P}\left(  \sup_{1\leq t \leq n} \varepsilon_t \geq \epsilon \right)
  &
  \leq \sum_{t=1}^n \mathbb{P}\left( \varepsilon_t \geq \epsilon \right) \leq  \sum_{t=1}^n \mathbb{P}\left( \varepsilon_t \geq \epsilon \right)
 =   \sum_{t=1}^n \mathbb{P}\big( \exp(\alpha \varepsilon_t) \geq \exp(\alpha \epsilon) \big)
  \\
  & \leq  \sum_{t=1}^n \mathbb{E}[\exp(\alpha \varepsilon_t - \alpha \epsilon) ]
  \leq \sum_{t=1}^n  \exp(\alpha^2 \varsigma^2/2- \alpha \epsilon) .
 \end{align*}
 For a fixed $\epsilon$, the function $\alpha \mapsto \alpha^2 \varsigma^2/2- \alpha \epsilon$ reaches its minimum when $\alpha = \epsilon / \varsigma^2$ and thus,
  \begin{equation} 
   \mathbb{P}\left(  \sup_{1\leq t \leq n} \varepsilon_t \geq \epsilon \right)
   \leq  n \exp\Big(-\frac{\epsilon^2}{2 \varsigma^2} \Big).
\end{equation}
So, it holds that,
 \begin{equation}\label{eq_proof_lem_reg}
   \mathbb{P}\left(  \sup_{1\leq t \leq n} |\varepsilon_t| \geq \epsilon \right)
   \leq 2 n \exp\Big(-\frac{\epsilon^2}{2 \varsigma^2} \Big).
\end{equation}
 Putting $\delta = 2 n \exp\big(-\epsilon^2/2\varsigma^2 \big)$, (\ref{eq_proof_lem_reg}) gives
 \begin{equation}\label{eq_proof_lem_reg_delta}
  \mathbb{P}\left(  \sup_{1\leq t \leq n} |\varepsilon_t| \geq \varsigma \sqrt{ 2 \log \frac{2n}{\delta} } \right) \leq \delta.
 \end{equation}
 Finally, since $h^*\in \mathcal{C}^{\beta,\mathcal{K}}(\mathcal{X})$, it holds a.s. that,  
 $|Y_t| \leq |h^*(\boldsymbol{X}_t)|+|\varepsilon_t| \leq \|h^*\|_{\infty, \mathcal{X} } + |\varepsilon_t|  \leq \mathcal{K} + |\varepsilon_t| $.
 Thus, the result follows from (\ref{eq_proof_lem_reg_delta}). 
\end{proof}
Therefore, consider a DNN class $\mathcal{H}_{\sigma,\dx,1} (L_n,N_n, B_n, F_n)$ for $L_n, N_n,  B_n>0$ and $F_n=F>\max(\mathcal{K}, 1)$.
From Lemma \ref{lem_reg}, it holds on $\tilde{\Omega}_{\delta}$ that all the variables $Y_t$ ($t \in \N$) belong to the compact set $[- \mathcal{K} -  \varsigma \sqrt{ 2\log \frac{2n}{\delta} },  \mathcal{K} +  \varsigma \sqrt{ 2\log \frac{2n}{\delta} }  ]$ and thus, for $h \in \mathcal{H}_{\sigma,\dx,1} (L_n,N_n,  B_n, F_n)$, $x \in \mx$ and $y \in [- \mathcal{K} -  \varsigma \sqrt{ 2\log \frac{2n}{\delta} },  \mathcal{K} +  \varsigma \sqrt{ 2\log \frac{2n}{\delta} }  ]$,
\begin{align*}
| \ell(h(\boldsymbol{x}),y)-\ell(h^*(\boldsymbol{x}),y) |
&
= | (h(\boldsymbol{x})-y)^2-(h^*(\boldsymbol{x})-y)^2 |
 \\
 & \leq | (2y-h(\boldsymbol{x})-h^*(\boldsymbol{x}))(h(\boldsymbol{x})-h^*(\boldsymbol{x})) |
 \\
 & \leq (2\mathcal{K} +  2\varsigma \sqrt{ 2\log \frac{2n}{\delta} } + F + \mathcal{K} )|h(\boldsymbol{x})-h^*(\boldsymbol{x})|,
\end{align*}
that is, we have Assumption (\textbf{A3}) with $\mathcal{K}_\ell=3\mathcal{K} +  2\varsigma \sqrt{ 2\log \frac{2n}{\delta} } + F $. Thus, on this set, we can apply Theorem \ref{theo_bound_comp} with $F_n=F>\max(\mathcal{K}, 1)$. Thus, in this setting, we obtain, with probability at least $1-2\delta$,
{\color{black}
 \begin{equation*}
   \mathcal{E}(\widehat{\boldsymbol{\theta}}_\lambda) \leq    \Xi_3 \Big( \phi_{n\gamma_n} (\log (n\gamma_n))^3 + \dfrac{3\mathcal{K} +  2\varsigma \sqrt{ 2\log \frac{2n}{\delta} } + F +\log(1/\delta)}{n\gamma_n} \Big) \leq   \Xi_4 \Big( \phi_{n\gamma_n} (\log (n\gamma_n)^3 + \dfrac{\log(n/\delta)}{n\gamma_n} \Big),
 \end{equation*}
 }
 for some $\Xi_4 >0$ independent of $n$.
Thus, when $\gamma_n\geq \underline{\gamma}>0$, the convergence rate of $\mathcal{E}(\widehat{\boldsymbol{\theta}}_\lambda)$ matches (up to a $\log^3 n$ factor) with the lower bound of \cite{schmidt2020nonparametric}.
Which shows that the Bayesian DNN estimator $\widehat{h}_{\lambda} = h_{\widehat{\boldsymbol{\theta}}_\lambda}$ for the least squares nonparametric regression {\color{black}with sub-Gaussian errors} is optimal in the minimax sense.

\subsection{Classification with the logistic Loss}
Consider a stationary and ergodic Markov process $(\boldsymbol{X}_t,Y_t)_{t \in \N}$ with distribution $P_{\boldsymbol{Z}_0}$ ($\boldsymbol{Z}_0 = (\boldsymbol{X}_0, Y_0)$) and with values in $\mx \times \{-1, 1\} \subset \R^\dx \times \R$ (that is, $\dy=1$) satisfying for all $\boldsymbol{x} \in \mx$,
\begin{equation}\label{mod_class}
 Y_t|\boldsymbol{X}_t=\boldsymbol{x} \sim 2\mathcal{B}\big( \eta(\boldsymbol{x}) \big) - 1, ~ \text{ with } \eta(\boldsymbol{x}) = \mathds{P}(Y_t = 1| \boldsymbol{X}_t = \boldsymbol{x}), 
\end{equation}
where $\mathcal{B}\big( \eta(\boldsymbol{x}) \big)$ denotes the Bernoulli distribution with parameter $\eta(\boldsymbol{x})$.
We focus on the margin-based logistic loss function given by,
$\ell(y, u) = \phi(y u)$ for all $(y,u) \in \{-1, 1\} \times \R$ and with $\phi(v) = \log(1 + e^{-v}), ~ \forall v \in \R$.
For any predictor $h : \R^\dx \rightarrow \R$, we have,
\[ R(h) = \E\big[ \phi\big(Y_0 h(\boldsymbol{X}_0)  \big)  \big] = \int_{\mx\times \{-1, 1\} } \phi\big( yh(\boldsymbol{x})  \big) dP_{\boldsymbol{Z}_0}(\boldsymbol{x},y) . \]
Here, we assume that $\mx$ is a compact set and deal with a class of bounded deep neural networks $\mathcal{H}_{\sigma,\dx,1} (L_n,N_n,  B_n, F_n, S_n)$ ($\dy=1$) with $F_n=F>\max(\mathcal{K}, 1)$.
So, the assumptions (\textbf{A1})-(\textbf{A3}) and (\textbf{A6}) with $\kappa=1$ are satisfied; see for instance \cite{bartlett2006convexity} in the case of (\textbf{A6}).

\medskip
Let us check (\textbf{A7}).
The target function, with respect to the $\phi$-loss is given by (see also Lemma \textbf{C.2} in \cite{zhang2024classification}) is the function $h^*: \mx \rightarrow [-\infty,\infty]$, defined by:
\begin{equation}\label{target_class}
h^*(\boldsymbol{x})=   
 \begin{cases}
   \infty & \text{if }  \eta(\boldsymbol{x}) = 1,  \\
    \log \frac{\eta(\boldsymbol{x})}{ 1 -\eta(\boldsymbol{x})} & \text{if }  \eta(\boldsymbol{x}) \in (0,1), \\
   -\infty  &   \text{if } \eta(\boldsymbol{x}) =0.
  \end{cases}
\end{equation}
Let $\boldsymbol{x} \in \mx$ such that $\eta(\boldsymbol{x}) \in (0,1)$.
For any measurable function $h : \R^\dx \rightarrow \R$, we have,
\begin{align}\label{cond_excess_risk_class}
\nonumber & \E\big[ \phi\big(Y_0 h(\boldsymbol{X}_0)  \big) | \boldsymbol{X}_0=\boldsymbol{x} \big] - \E\big[ \phi\big(Y_0 h^*(\boldsymbol{X}_0)  \big) | \boldsymbol{X}_0=\boldsymbol{x} \big] \\
\nonumber &= \eta(\boldsymbol{x}) \phi\big( h(\boldsymbol{x}) \big) + \big(1- \eta(\boldsymbol{x}) \big) \phi\big( -h(\boldsymbol{x}) \big) - \eta(\boldsymbol{x}) \phi\big( h^*(\boldsymbol{x}) \big) - \big(1- \eta(\boldsymbol{x}) \big) \phi\big( -h^*(\boldsymbol{x}) \big) \\
\nonumber &= \eta(\boldsymbol{x}) \phi\big( h(\boldsymbol{x}) \big) + \big(1- \eta(\boldsymbol{x}) \big) \phi\big( -h(\boldsymbol{x}) \big) - \eta(\boldsymbol{x}) \log \frac{1}{\eta(\boldsymbol{x})} - \big(1- \eta(\boldsymbol{x}) \big) \log \frac{1}{1-\eta(\boldsymbol{x})} \\
& \leq \frac{1}{8} \Big|h(\boldsymbol{x}) - \log \frac{\eta(\boldsymbol{x})}{ 1 -\eta(\boldsymbol{x})} \Big|^2  \leq \frac{1}{8} |h(\boldsymbol{x}) - h^*(\boldsymbol{x}) |^2,
\end{align}
where the inequality (\ref{cond_excess_risk_class}) holds from the Lemma C.6 in \cite{zhang2024classification}.
With the convention $0 \times \infty = 0$, the inequality (\ref{cond_excess_risk_class}) holds for $\eta(\boldsymbol{x})=1$ or $\eta(\boldsymbol{x})=0$.  
Hence,
\begin{equation*}
R(h) - R(h^*) = \E\big[ \phi\big(Y_0 h(\boldsymbol{X}_0)  \big)  \big] - \E\big[ \phi\big(Y_0 h^*(\boldsymbol{X}_0)  \big)  \big] \leq  \frac{1}{8} \E[|h(\boldsymbol{X}_0) - h^*(\boldsymbol{X}_0) |^2] =  \frac{1}{8} \| h - h^*\|^2_{2, P_{\boldsymbol{X}_0}}.
\end{equation*}
Thus, (\textbf{A7}) holds with $\mk_0 = \frac{1}{8}$.
Therefore:
 \begin{itemize}
 \item If (\textbf{A4}) holds and $h^* \in \mathcal{C}^{\beta, \mathcal{K}}(\mx)$ for some $\beta, \mk>0$, then the result of Theorem \ref{theo_bound_Holder} is applied.
 \item For $\mx = [0,1]^\dx$, if $h^* \in \mathcal{G}(q, \boldsymbol{d}, \boldsymbol{t}, \boldsymbol{\beta}, \mathcal{K})$ for some composition structured class, then the result of Theorem \ref{theo_bound_comp} is applicable with the ReLU deep neural networks.
\end{itemize}   

The next theorem provides a lower bound of the excess risk on a class of composition structured functions. 
\begin{thm}[Lower bound for classification]\label{theo_lower_bound_class}
 Set $\mx = [0,1]^\dx$ and assume that $(\boldsymbol{X}_t,Y_t)_{t \in \N}$ satisfying (\ref{mod_class}) is a stationary and ergodic Markov process. 
Consider a class of composition structured functions $\mathcal{G}(q, \bold{d}, \bold{t}, \boldsymbol{\beta}, \mathcal{K})$ with $t_i \leq \min(d_0, \dots, d_{i-1})$ for all $i$.
 Then, for sufficiently large $\mk$, there exist a constant $C>0$  such that 
\begin{equation}\label{equa_lower_bound}
\underset{\widehat{h}_n}{\inf}~\underset{h^{*}\in\mathcal{G}(q, \bold{d}, \bold{t}, \boldsymbol{\beta}, \mk)}{\sup} \E\big[   R(\widehat{h}_n) - R(h^*) \big]  \ge C\phi_{n},
\end{equation}
where the infimum is taken over all estimator $\widehat{h}_n$, based on $(\boldsymbol{X}_1,Y_1),\cdots, (\boldsymbol{X}_n,Y_n)$ generated from (\ref{mod_class}) and $\phi_{n}$ is defined in (\ref{def_beta_star_phi_n}).
\end{thm}

\medskip
\noindent
Theorem \ref{theo_bound_comp} and \ref{theo_lower_bound_class} show that, for the classification problem with the logistic loss of Markov processes {\color{black}satisfying $\gamma_n\geq \underline{\gamma}>0$)}, the quasi-Bayesian DNN predictor is optimal (up to a $\log^3 n$ factor) in the minimax sense, on the class of composition structured functions.

\section{Proofs of the main results}
\label{sec:proofs}
\subsection{Some technical lemmas}
The following lemmas are useful in the proofs of the main results.
\begin{lem}\label{lem_Paulin} (Paulin \cite{paulin2015concentration}).
Consider a trajectory $(\boldsymbol{Z}_1,\cdots,\boldsymbol{Z}_n)$ of a process $Z=(\boldsymbol{Z}_t)_{t \in \N}$ with values in $\mathcal{Z}$ and assume that (\textbf{A1}) holds.
Let $f \in L^2(\pi)$ such that, there exists $C>0$ satisfying $|f(z) - \E_\pi (f) | \leq C$ for all $z \in \mathcal{Z}$.
Set $V_f:=\text{Var}_\pi (f)$ and $S=\sum_{i=1}^n f(\boldsymbol{Z}_i)$. 
Then, for all $\vartheta \in \big(0, \gamma_{n} /10 \big)$, we have,
\begin{equation}\label{exp_inq_Paulin}
\E_\pi \big( \exp(\vartheta S) \big) \leq \exp\Bigg( \frac{2 (n+1) V_f }{ \gamma_{n}} \cdot \vartheta^2 \cdot \Big(1 - \frac{10 \vartheta}{ \gamma_{n} } \Big)^{-1}  \Bigg).
\end{equation}
\end{lem}
\begin{lem}\label{lem_Bernstein_type_ineq}
Assume that (\textbf{A1}), (\textbf{A3}), (\textbf{A5}) and (\textbf{A6}) hold. 
Consider the DNNs class $\mathcal{H}_{\sigma,\dx,\dy} (L,N, B, F, \mathcal{I})$ for $L,N, B, F >0$, $\mathcal{I} \subseteq \{1,2,\cdots, n_{L,N}\}$. 
For all $\lambda \in  \big(0, n \cdot \gamma_{n} /10 \big)$ and $ \theta \in \mathcal{S}_{\mathcal{I}}$ we have,
 \begin{equation}\label{Bernstein_type_ineq}
 \max\Big( \E\big[\exp \big( \lambda \big(  \widehat{\mathcal{E}}_n(\boldsymbol{\theta}) - \mathcal{E}(\boldsymbol{\theta}) \big)  \big)   \big] ,  \E\big[\exp \big( \lambda \big( \mathcal{E}(\boldsymbol{\theta}) - \widehat{\mathcal{E}}_n(\boldsymbol{\theta})  \big)  \big)   \big]  \Big) \leq  \exp\Bigg( \frac{ 16 K \lambda^2 \mathcal{E}(\boldsymbol{\theta})^\kappa  }{ n \gamma_{n}}  \Big(1 - \frac{10 \lambda}{ n \gamma_{n} } \Big)^{-1}  \Bigg), 
\end{equation} 
where $K, \kappa$ are given in assumption (\textbf{A6}). 
\end{lem}
\begin{dem}
Let $ \theta \in \mathcal{S}_{\mathcal{I}}$. We have,
\begin{align*}
\widehat{\mathcal{E}}_n(\boldsymbol{\theta}) - \mathcal{E}(\boldsymbol{\theta}) 
&= \widehat{R}_n(\boldsymbol{\theta}) - \widehat{R}_n(h^*) - R(\boldsymbol{\theta}) + R(h^*)  \\
& = \frac{1}{n} \sum_{i=1}^n \Big( \ell\big(h_{\boldsymbol{\theta}}(\boldsymbol{X_i}),\boldsymbol{Y_i} \big) - \ell\big(h^*(\boldsymbol{X_i}),\boldsymbol{Y_i} \big) - \E_\pi [\ell\big(h_{\boldsymbol{\theta}}(\boldsymbol{X}_1),\boldsymbol{Y}_1 \big) ]  +  \E_\pi [\ell\big(h^*(\boldsymbol{X}_1),\boldsymbol{Y}_1 \big) ] \Big) = \frac{1}{n} \sum_{i=1}^n g_\theta(\boldsymbol{Z}_i),
\end{align*} 
with $\boldsymbol{Z}_i=(\boldsymbol{X}_i,\boldsymbol{Y}_i)$, $g_\theta(\boldsymbol{z})=\ell\big(h_{\boldsymbol{\theta}}(\boldsymbol{x}),\boldsymbol{y} \big) - \ell\big(h^*(\boldsymbol{x}),\boldsymbol{y} \big) - \E_\pi \big[\ell\big(h_{\boldsymbol{\theta}}(\boldsymbol{X}_1),\boldsymbol{Y}_1 \big) \big]  +  \E_\pi \big[\ell\big(h^*(\boldsymbol{X}_1),\boldsymbol{Y}_1 \big) \big] $ for all $\boldsymbol{z}=(\boldsymbol{x},\boldsymbol{y}) \in \mz=\mx \times \my$.
We will apply Lemma \ref{lem_Paulin} to the process $(\boldsymbol{Z}_t)_{t \in \N}$ with the function $g_\theta$, satisfying $\E_\pi(g_\theta)=0$. 
According to the Bernstein's condition, we have,
\begin{align} 
\nonumber \text{Var}_\pi[g_\theta(\boldsymbol{Z}_1)] &= \E_\pi\big[ \big( g_\theta(\boldsymbol{Z}_1) \big)^2 \big] \\
\nonumber & \leq 2 \Big( \E_\pi \big[ \big( \ell\big(h_{\boldsymbol{\theta}}(\boldsymbol{X}_1),\boldsymbol{Y}_1 \big) - \ell\big(h^*(\boldsymbol{X}_1),\boldsymbol{Y}_1 ) \big)^2 \big] + \big( \E_\pi [\ell\big(h_{\boldsymbol{\theta}}(\boldsymbol{X}_1),\boldsymbol{Y}_1 \big) ]  - \E_\pi [\ell\big(h^*(\boldsymbol{X}_1),\boldsymbol{Y}_1 \big) ] \big)^2 \Big) \\
\nonumber &\leq 4 \E_\pi \big[ \big( \ell\big(h_{\boldsymbol{\theta}}(\boldsymbol{X}_1),\boldsymbol{Y}_1 \big) - \ell\big(h^*(\boldsymbol{X}_1),\boldsymbol{Y}_1 ) \big)^2 \big] 
 \leq 4K \mathcal{E}(\boldsymbol{\theta})^\kappa.
\end{align}
So, for any $\vartheta \in \big(0, \gamma_{n} /10 \big)$, it holds from Lemma \ref{lem_Paulin} that,    
\begin{equation}\label{exp_inq_Paulin_sum_g}
\E_\pi \Big[ \exp\Big(\vartheta \sum_{i=1}^n g_\theta(\boldsymbol{Z}_i) \Big) \Big] \leq \exp\Bigg( \frac{ 16 K n \mathcal{E}(\boldsymbol{\theta})^\kappa  }{ \gamma_{n}} \cdot \vartheta^2 \cdot \Big(1 - \frac{10 \vartheta}{ \gamma_{n} } \Big)^{-1}  \Bigg).
\end{equation}
Hence, for any $\lambda \in  \big(0, n \cdot \gamma_{n} /10 \big)$, we get,
\begin{equation}\label{exp_inq_Paulin_hat_E_theta}
 \E_\pi \big[\exp \big( \lambda \big(  \widehat{\mathcal{E}}_n(\boldsymbol{\theta}) - \mathcal{E}(\boldsymbol{\theta}) \big)  \big)   \big] = \E_\pi \Big[ \exp\Big( \frac{\lambda}{n} \sum_{i=1}^n g_\theta(\boldsymbol{Z}_i) \Big) \Big] \leq \exp\Bigg( \frac{ 16 \lambda^2 K \mathcal{E}(\boldsymbol{\theta})^\kappa  }{ n \gamma_{n}}  \Big(1 - \frac{10 \lambda}{ n \gamma_{n} } \Big)^{-1}  \Bigg).
\end{equation}
Similar arguments yield,
\begin{equation*} 
 \E_\pi \big[\exp \big( \lambda \big( \mathcal{E}(\boldsymbol{\theta}) -  \widehat{\mathcal{E}}_n(\boldsymbol{\theta})  \big)  \big)   \big] \leq \exp\Bigg( \frac{ 16\lambda^2 K \mathcal{E}(\boldsymbol{\theta})^\kappa  }{ n \gamma_{n}}  \Big(1 - \frac{10 \lambda}{ n \gamma_{n} } \Big)^{-1}  \Bigg),
\end{equation*}
for all $\lambda \in  \big(0, n \cdot \gamma_{n} /10 \big)$; which completes the proof of the lemma.
\end{dem}

 Let $(E,\mathcal{A})$ be a measurable space and $\mu_1$, $\mu_2$  two probability measures on this space.
Define the Kullback-Leibler divergence of $\mu_1$ with respect to $\mu_2$ by,

\[ \KL(\mu_1, \mu_2) =  \left\{
\begin{array}{ll}
\displaystyle{ \int \log\Big( \dfrac{d \mu_1}{ d\mu_2}\Big) d\mu_1   }   &  \text{ if } ~  \mu_1 \ll \mu_2, \\
 \infty &   ~ \text{ otherwise},
\end{array}
\right.  
  \]
where the notation $\mu_1 \ll \mu_2$ means "$\mu_1$ is absolutely continuous with respect to $\mu_2$".
The following lemma is classical for PAC-Bayes bounds, the proof can bee found in \cite{catoni2007pac}.
\begin{lem}\label{lem_log_exp_KL}
Let $(E,\mathcal{A})$ be a measurable space, $\mu$ a probability measure on $(E,\mathcal{A})$ and $h: E \rightarrow \R$ a measurable function such that $\int \exp \circ h ~d \mu < \infty$.
 With the convention $\infty - \infty = -\infty$, we have,
 \begin{equation}\label{lem_log_exp_KL_eq}
  \log\Big( \int \exp \circ h ~d \mu \Big) = \sup_{\nu} \Big( \int h d\nu - \KL(\nu,\mu) \Big), 
\end{equation} 
 where the supremum is taken over all probability measures $\nu$ on $(E,\mathcal{A})$.
 In addition, if $h$ is bounded from above on the support of $\mu$, the supremum in (\ref{lem_log_exp_KL_eq}) is reached for the Gibbs distribution $\nu =g$, with $\frac{d g}{d \mu} \propto \exp \circ h$.
\end{lem} 
The following lemmas are also needed.
 \begin{lem}\label{lem_PAC_bound}
Assume that (\textbf{A1}), (\textbf{A3}), (\textbf{A5}) and (\textbf{A6}) with $\kappa=1$ hold and consider the DNNs class $\mathcal{H}_{\sigma,\dx,\dy} (L,N, B, F, \mathcal{I})$ for $L,N, B, F >0$, $\mathcal{I} \subseteq \{1,2,\cdots, n_{L,N}\}$.
Let $\rho$ be a $D_n$-dependent probability measure on $\mathcal{S}_{\mathcal{I}}$ such that $\rho \ll \Pi$.
For any $\lambda \in  \Big(0,  \dfrac{ n \gamma_{n} }{16K + 10} \Big)$ and $\delta \in (0,1)$, we have with probability at least $1-\delta$,
 \begin{equation}\label{PAC_bound}
   \mathcal{E}(\widehat{\boldsymbol{\theta}}) \leq  \dfrac{1}{  1 - \dfrac{16 K \lambda}{  n \gamma_{n} - 10 \lambda  }    }  \Bigg[ \widehat{\mathcal{E}}_n(\widehat{\boldsymbol{\theta}})  + \dfrac{1}{\lambda} \Bigg( \log(1/\delta)  + \log\Big( \dfrac{d\rho}{ d \Pi}(\widehat{\boldsymbol{\theta}}) \Big) \Bigg) \Bigg].
 \end{equation}
 where $\widehat{\boldsymbol{\theta}}$ is distributed according to $\rho$.
\end{lem}
 
\begin{dem}
Let $ \theta \in \mathcal{S}_{\mathcal{I}}$, $\lambda \in  \big(0, n \cdot \gamma_{n} /10 \big)$ and $\delta \in ]0,1[$. We have from Lemma \ref{lem_Bernstein_type_ineq}, 
 \begin{equation*} 
  \E\big\{ \exp \big[ \lambda \big( \mathcal{E}(\boldsymbol{\theta}) - \widehat{\mathcal{E}}_n(\boldsymbol{\theta})  \big)  \big]   \big \} 
  \leq  \exp\Bigg( \frac{ 16 \lambda^2 K \mathcal{E}(\boldsymbol{\theta}) }{ n \gamma_{n}}  \Big(1 - \frac{10 \lambda}{ n \gamma_{n} } \Big)^{-1}  \Bigg). 
\end{equation*} 
Therefore,
 \begin{equation}\label{proof_lem_bound_delta} 
  \E\Bigg\{ \exp \Bigg[ \Big( \lambda - \dfrac{16 K \lambda^2}{ n \gamma_{n} \big( 1 - \frac{10 \lambda}{ n \gamma_{n} } \big) } \Big) \mathcal{E}(\boldsymbol{\theta}) - \lambda \widehat{\mathcal{E}}_n(\boldsymbol{\theta})    - \log(1/\delta) \Bigg]   \Bigg \} 
  \leq  \delta. 
\end{equation} 
Let us recall that $\Pi$  is a prior probability measure on $\mathcal{S}_{\mathcal{I}}$.
From (\ref{proof_lem_bound_delta}), we have
 \begin{equation*} 
 \int \E\Bigg\{ \exp \Bigg[ \Big( \lambda - \dfrac{16 K \lambda^2}{ n \gamma_{n} \big( 1 - \frac{10 \lambda}{ n \gamma_{n} } \big) } \Big) \mathcal{E}(\boldsymbol{\theta}) - \lambda \widehat{\mathcal{E}}_n(\boldsymbol{\theta})    - \log(1/\delta) \Bigg]   \Bigg \} d\Pi(\boldsymbol{\theta}) 
  \leq  \delta, 
\end{equation*} 
and by the Fubini's theorem,
 \begin{equation*}  
  \E\Bigg\{ \int \exp \Bigg[ \Big( \lambda - \dfrac{16 K \lambda^2}{ n \gamma_{n} \big( 1 - \frac{10 \lambda}{ n \gamma_{n} } \big) } \Big) \mathcal{E}(\boldsymbol{\theta}) - \lambda \widehat{\mathcal{E}}_n(\boldsymbol{\theta})    - \log(1/\delta) \Bigg] d\Pi(\boldsymbol{\theta})  \Bigg \}  
  \leq  \delta. 
\end{equation*} 
Hence, for any data-dependent probability measure $\rho$ absolutely continuous with respect to $\Pi$, we have
 \begin{equation}\label{proof_lem_bound_delta_rho} 
  \E\Bigg\{ \int \exp \Bigg[ \Big( \lambda - \dfrac{16 K \lambda^2}{ n \gamma_{n} \big( 1 - \frac{10 \lambda}{ n \gamma_{n} } \big) } \Big) \mathcal{E}(\boldsymbol{\theta}) - \lambda \widehat{\mathcal{E}}_n(\boldsymbol{\theta})    - \log(1/\delta) - \log\Big( \dfrac{d\rho}{ d \Pi}(\boldsymbol{\theta}) \Big) \Bigg] d\rho(\boldsymbol{\theta})  \Bigg \}  
  \leq  \delta, 
\end{equation} 
with the convention $\infty \times 0 = 0$.
Recall that, $P_n$ denotes the distribution of the sample $D_n$.
The inequality (\ref{proof_lem_bound_delta_rho}) can also been written as,
 \begin{equation*}
  \E_{D_n \sim P_n} \E_{\widehat{\boldsymbol{\theta}} \sim \rho} \Bigg\{  \exp \Bigg[ \Big( \lambda - \dfrac{16 K \lambda^2}{ n \gamma_{n} \big( 1 - \frac{10 \lambda}{ n \gamma_{n} } \big) } \Big) \mathcal{E}(\widehat{\boldsymbol{\theta}}) - \lambda \widehat{\mathcal{E}}_n(\widehat{\boldsymbol{\theta}})    - \log(1/\delta) - \log\Big( \dfrac{d\rho}{ d \Pi}(\widehat{\boldsymbol{\theta}}) \Big) \Bigg]  \Bigg \}
  \leq  \delta, 
\end{equation*} 
Therefore, from the inequality $\exp(\lambda x) \geq \ind_{x \geq 0}$, we get with $P_n \otimes \rho$-probability at most $\delta$,
 \begin{equation*}
     \Big( 1 - \dfrac{16 K \lambda}{ n \gamma_{n} \big( 1 - \frac{10 \lambda}{ n \gamma_{n} } \big) } \Big) \mathcal{E}(\widehat{\boldsymbol{\theta}}) \geq   \widehat{\mathcal{E}}_n(\widehat{\boldsymbol{\theta}})  + \dfrac{1}{\lambda} \Bigg( \log(1/\delta)  + \log\Big( \dfrac{d\rho}{ d \Pi}(\widehat{\boldsymbol{\theta}}) \Big) \Bigg).  
\end{equation*}
Since $\lambda \in  \Big(0, \dfrac{n \gamma_{n} }{16 K + 10} \Big)$, we have with $P_n \otimes \rho$-probability at least $1-\delta$,
 \begin{equation}\label{proof_lem_bound_at_least} 
      \mathcal{E}(\widehat{\boldsymbol{\theta}}) \leq  \dfrac{1}{  1 - \dfrac{16 K \lambda}{  n \gamma_{n} - 10 \lambda  }    }  \Bigg[ \widehat{\mathcal{E}}_n(\widehat{\boldsymbol{\theta}})  + \dfrac{1}{\lambda} \Bigg( \log(1/\delta)  + \log\Big( \dfrac{d\rho}{ d \Pi}(\widehat{\boldsymbol{\theta}}) \Big) \Bigg) \Bigg].  
\end{equation}
This completes the proof of the lemma.
\end{dem}

\begin{lem}\label{lem_PAC_bound_KL}
Assume that (\textbf{A1}), (\textbf{A3}), (\textbf{A5}) and (\textbf{A6}) with $\kappa=1$ hold and consider the DNNs class $\mathcal{H}_{\sigma,\dx,\dy} (L,N, B, F, \mathcal{I})$ for $L,N, B, F >0$, $\mathcal{I} \subseteq \{1,2,\cdots, n_{L,N}\}$.
Let $\rho$ be a $D_n$-dependent probability measure on $\mathcal{S}_{\mathcal{I}}$ such that $\rho \ll \Pi$.
For any $\lambda \in  \Big(0, \dfrac{n \gamma_{n} }{16 K + 10} \Big)$ and $\delta \in (0,1)$, we have with probability at least $1-\delta$,
 \begin{equation}\label{PAC_bound_KL}
   \mathcal{E}(\widehat{\boldsymbol{\theta}}_\lambda) \leq  \dfrac{1}{  1 - \dfrac{16 K \lambda}{  n \gamma_{n} - 10 \lambda  }    }   \Bigg[ \Big( 1 + \dfrac{16 K \lambda}{ n \gamma_{n} - 10 \lambda} \Big) \int \mathcal{E}(\boldsymbol{\theta})  d \rho(\boldsymbol{\theta}) + \dfrac{2}{\lambda} \big( \KL(\rho, \Pi) + \log(1/\delta) \big)   \Bigg],
 \end{equation}
 where $\widehat{\boldsymbol{\theta}}_\lambda \sim \Pi_\lambda (\cdot | D_n)$.
\end{lem}
 
\begin{dem}
Let $\lambda \in  \Big(0, \dfrac{n \gamma_{n} }{16 K + 10} \Big)$ and $\delta \in (0,1)$.
Apply Lemma \ref{lem_PAC_bound} for the posterior distribution $\Pi_\lambda (\cdot | D_n) \ll \Pi$ with the corresponding Radon-Nikodym density
\begin{equation}\label{lem_eq_den}
\dfrac{d \Pi_\lambda ( \boldsymbol{\theta}| D_n)   }{ d \Pi( \boldsymbol{\theta}) } = \dfrac{ \exp\big(-\lambda \widehat{R}_n(\boldsymbol{\theta}) \big) }{  \int \exp\big(-\lambda \widehat{R}_n(\boldsymbol{\theta}) \big) d \Pi(\boldsymbol{\theta}) }.
\end{equation}
So, in addition to Lemma \ref{lem_log_exp_KL}, we have with probability at least $1-\delta$,
 \begin{align}\label{proof_lem_PAC_bound_KL_int}
  \nonumber \mathcal{E}(\widehat{\boldsymbol{\theta}}_\lambda) &\leq  \dfrac{1}{  1 - \dfrac{16 K \lambda}{  n \gamma_{n} - 10 \lambda  }    }  \Bigg[ \widehat{\mathcal{E}}_n(\widehat{\boldsymbol{\theta}}_\lambda)  + \dfrac{1}{\lambda} \Bigg(   \log(1/\delta)  - \lambda \widehat{R}_n( \widehat{\boldsymbol{\theta}}_\lambda ) - \log \int \exp\big(-\lambda \widehat{R}_n(\boldsymbol{\theta}) \big) d \Pi(\boldsymbol{\theta})    \Bigg) \Bigg] \\
\nonumber &\leq  \dfrac{1}{  1 - \dfrac{16 K \lambda}{  n \gamma_{n} - 10 \lambda  }    }  \Bigg[ \widehat{R}_n( h^* )  + \dfrac{1}{\lambda} \Bigg(   \log(1/\delta)  - \log \int \exp\big(-\lambda \widehat{R}_n(\boldsymbol{\theta}) \big) d \Pi(\boldsymbol{\theta})    \Bigg) \Bigg] \\ 
\nonumber &\leq  \dfrac{1}{  1 - \dfrac{16 K \lambda}{  n \gamma_{n} - 10 \lambda  }    }  \Bigg[ \widehat{R}_n( h^* )  + \dfrac{1}{\lambda} \Bigg(   \log(1/\delta)  - \sup_{\rho,\rho \ll \Pi} \Big[ -\lambda \int \widehat{R}_n(\boldsymbol{\theta}) d \rho(\boldsymbol{\theta}) - \KL(\rho,\Pi) \Big]  \Bigg) \Bigg]  \\
\nonumber &\leq  \dfrac{1}{  1 - \dfrac{16 K \lambda}{  n \gamma_{n} - 10 \lambda  }    }  \inf_{\rho,\rho \ll \Pi}  \Bigg[ \widehat{R}_n( h^* )  + \dfrac{1}{\lambda} \Bigg(   \log(1/\delta)  + \lambda \int \widehat{R}_n(\boldsymbol{\theta}) d \rho(\boldsymbol{\theta}) + \KL(\rho,\Pi) \Bigg) \Bigg] \\
&\leq  \dfrac{1}{  1 - \dfrac{16 K \lambda}{  n \gamma_{n} - 10 \lambda  }    }  \inf_{\rho,\rho \ll \Pi}  \Bigg[  \int \ \widehat{\mathcal{E}}_n(\boldsymbol{\theta}) d \rho(\boldsymbol{\theta})  + \dfrac{1}{\lambda} \Bigg(   \log(1/\delta)  + \KL(\rho,\Pi) \Bigg) \Bigg]
 \end{align}
Let $\rho$ a probability measure on $\mathcal{S}_{\mathcal{I}}$ such that $\rho \ll \Pi$.
From Lemma \ref{lem_Bernstein_type_ineq}, we have
 \begin{equation*} 
  \E\big\{ \exp \big[ \lambda \big( \widehat{\mathcal{E}}_n(\boldsymbol{\theta}) -  \mathcal{E}(\boldsymbol{\theta})   \big)  \big]   \big \} 
  \leq  \exp\Bigg( \frac{ 16 K \lambda^2 \mathcal{E}(\boldsymbol{\theta}) }{ n \gamma_{n}}  \Big(1 - \frac{10 \lambda}{ n \gamma_{n} } \Big)^{-1}  \Bigg). 
\end{equation*} 
Therefore,
 \begin{equation}\label{proof_lem_bound_delta_KL} 
  \E\Bigg\{ \exp \Bigg[\lambda \widehat{\mathcal{E}}_n(\boldsymbol{\theta})  - \lambda \Big( 1 + \dfrac{16 K \lambda}{ n \gamma_{n} - 10 \lambda} \Big) \mathcal{E}(\boldsymbol{\theta})     - \log(1/\delta) \Bigg]   \Bigg \} 
  \leq  \delta, 
\end{equation}
and from the Fubini's theorem,
 \begin{equation}\label{proof_lem_bound_delta_KL_Fub} 
  \E\Bigg\{ \int \exp \Bigg[\lambda \widehat{\mathcal{E}}_n(\boldsymbol{\theta})  - \lambda \Big( 1 + \dfrac{16 K \lambda}{ n \gamma_{n} - 10 \lambda} \Big) \mathcal{E}(\boldsymbol{\theta})     - \log(1/\delta) \Bigg]  d \Pi(\boldsymbol{\theta}) \Bigg \} 
  \leq  \delta. 
\end{equation}
Hence, in addition to the Jensen's inequality,
\begin{align*}
& \E_{D_n} \Bigg[  \exp\Bigg( \int \Big[ \lambda \widehat{\mathcal{E}}_n(\boldsymbol{\theta})  - \lambda \Big( 1 + \dfrac{16 K \lambda}{ n \gamma_{n} - 10 \lambda} \Big) \mathcal{E}(\boldsymbol{\theta}) \Big] d \rho(\boldsymbol{\theta})  -\KL(\rho, \Pi)   - \log(1/\delta)   \Bigg) \Bigg]   \\
& = \E_{D_n} \Bigg[  \exp\Bigg( \int \Big[ \lambda \widehat{\mathcal{E}}_n(\boldsymbol{\theta})  - \lambda \Big( 1 + \dfrac{16 K \lambda}{ n \gamma_{n} - 10 \lambda} \Big) \mathcal{E}(\boldsymbol{\theta})   - \dfrac{d\rho(\boldsymbol{\theta}) }{ d\Pi(\boldsymbol{\theta})}   - \log(1/\delta)  \Big] d \rho(\boldsymbol{\theta}) \Bigg) \Bigg] \\
& \leq \E_{D_n, \theta \sim \rho} \Bigg[  \exp\Bigg(   \lambda \widehat{\mathcal{E}}_n(\boldsymbol{\theta})  - \lambda \Big( 1 + \dfrac{16 K \lambda}{ n \gamma_{n} - 10 \lambda} \Big) \mathcal{E}(\boldsymbol{\theta})   - \dfrac{d\rho(\boldsymbol{\theta}) }{ d\Pi(\boldsymbol{\theta})}   - \log(1/\delta) \Bigg) \Bigg] \\
& \leq \E_{D_n } \Bigg[ \int \exp\Bigg(   \lambda \widehat{\mathcal{E}}_n(\boldsymbol{\theta})  - \lambda \Big( 1 + \dfrac{16 K \lambda}{ n \gamma_{n} - 10 \lambda} \Big) \mathcal{E}(\boldsymbol{\theta})   - \log(1/\delta) \Bigg) d \Pi(\boldsymbol{\theta}) \Bigg] \leq \delta.    
\end{align*}
Therefore, from the inequality $\exp(\lambda x) \geq \ind_{x \geq 0}$, we get with probability at least $1-\delta$,
\begin{equation}\label{proof_lem_bound_delta_KL_Jens}
\int \ \widehat{\mathcal{E}}_n(\boldsymbol{\theta}) d \rho(\boldsymbol{\theta}) \leq \Big( 1 + \dfrac{16 K \lambda}{ n \gamma_{n} - 10 \lambda} \Big) \int \mathcal{E}(\boldsymbol{\theta})  d \rho(\boldsymbol{\theta}) + \dfrac{1}{\lambda} \big( \KL(\rho, \Pi) + \log(1/\delta) \big).
\end{equation}
According to (\ref{proof_lem_PAC_bound_KL_int}) and (\ref{proof_lem_bound_delta_KL_Jens}), we get we with probability at least $1-2\delta$,
\begin{align*} 
  \nonumber \mathcal{E}(\widehat{\boldsymbol{\theta}}_\lambda) 
&\leq  \dfrac{1}{  1 - \dfrac{16 K \lambda}{  n \gamma_{n} - 10 \lambda  }    }  \inf_{\rho,\rho \ll \Pi}  \Bigg[ \Big( 1 + \dfrac{16 K \lambda}{ n \gamma_{n} - 10 \lambda} \Big) \int \mathcal{E}(\boldsymbol{\theta})  d \rho(\boldsymbol{\theta}) + \dfrac{2}{\lambda} \big( \KL(\rho, \Pi) + \log(1/\delta) \big)   \Bigg] \\
&=  \dfrac{1}{  1 - \dfrac{16 K \lambda}{  n \gamma_{n} - 10 \lambda  }    }  \inf_{\rho,\rho \ll \Pi}  \Bigg[ \Big( 1 + \dfrac{16 K \lambda}{ n \gamma_{n} - 10 \lambda} \Big) \int \mathcal{E}(\boldsymbol{\theta})  d \rho(\boldsymbol{\theta}) + \dfrac{2}{\lambda} \big( \KL(\rho, \Pi) + \log(1/\delta) \big)   \Bigg]
 \end{align*}
\end{dem}
%
%

 \begin{lem}\label{lem_lip_theta}
Assume that (\textbf{A4}) and consider the DNNs class $\mathcal{H}_{\sigma,\dx,\dy} (L,N, B, F, \mathcal{I})$ for $L,N, B, F >0$, $\mathcal{I} \subseteq \{1,2,\cdots, n_{L,N}\}$. 
For all $x \in \R^\dx$, $\boldsymbol{\theta}, \widetilde{\boldsymbol{\theta}} \in \mathcal{S}_{\mathcal{I}}$, we have,
\[ |h_{\boldsymbol{\theta}}(x) - h_{\widetilde{\boldsymbol{\theta}}} (x) | \leq  2 L^2 \big(|\sigma(0)| + \| x \|  +1 \big) \big( 1 +  C_{\sigma} B \big) \max\big(1, (C_{\sigma} B)^{2L} \big) \| \boldsymbol{\theta} - \widetilde{\boldsymbol{\theta}} \|  .\] 
\end{lem}
 
\begin{dem}
In the sequel, we use the following elementary result: Let $a \geq 0$ and $(b_m)_{m \in \N}$, $(u_m)_{m \in \N}$ two sequences of non negative numbers such that $u_m \leq a u_{m-1} + b_{m-1}$ for all $m \geq 1$. 
Then, for all $m \geq 1$,
\begin{equation}\label{seq_recur}
u_m \leq  a^{m} u_0 + \sum_{k=1}^{m} b_{m-k} a^{k-1}.
\end{equation}

\medskip
 Consider two neural networks $h_{\boldsymbol{\theta}}$, $h_{\widetilde{\boldsymbol{\theta}}}$ as in (\ref{DNN_def}) with $A_\ell(x) = W_\ell x + b_\ell$, $\widetilde{A}_\ell(x) = \widetilde{W}_\ell x + \widetilde{b}_\ell$ for all $\ell=1,\cdots,L+1$ and $\boldsymbol{\theta} = \big( \text{vec}(W_1)^T, b_1^T,\cdots, \text{vec}(W_{L+1})^T, b_{L+1}^T \big), \widetilde{\boldsymbol{\theta}} = \big( \text{vec}(\widetilde{W}_1)^T, \widetilde{b}_1^T,\cdots, \text{vec}(\widetilde{ W}_{L+1})^T, \widetilde{b}_{L+1}^T \big) \in \mathcal{S}_{\mathcal{I}}$ (see also (\ref{def_theta})). Let $x \in \R^\dx$. Set,
\begin{equation}\label{def_x_ell}
 x^{(0)} = \widetilde{x}^{(0)} = x, ~     x^{(\ell)} = \sigma_\ell \circ A_\ell(x^{(\ell-1)})  \text{ and }  \widetilde{x}^{(\ell)} = \sigma_\ell \circ \widetilde{A}_\ell(\widetilde{x}^{(\ell-1)})   \text{ for any } \ell=1,\cdots, L+1.
\end{equation} 
According to assumption (\textbf{A4}), we have for $\ell=1,\cdots, L$,
\begin{align*}
\| x^{(\ell)} \| &= \|\sigma_1 \circ A_\ell(x^{(\ell-1)}) \| \leq |\sigma(0)| + C_{\sigma} \| A_\ell( x^{(\ell-1)} ) \| \\
& \leq  |\sigma(0)| + C_{\sigma}  \big( \| W_\ell\| \| x^{(\ell-1)}\| + \| b_\ell \| \big) \leq  |\sigma(0)| + C_{\sigma} B + C_{\sigma} B \| x^{(\ell-1)}\|.         
\end{align*} 
Therefore, in addition to (\ref{seq_recur}), we have for all $\ell=1,\cdots, L$,
\begin{align} \label{eq_norm_x_ell}
\nonumber \| x^{(\ell)} \| &\leq (C_{\sigma} B)^\ell \| x\| + (|\sigma(0)| + C_{\sigma} B) \sum_{k=1}^\ell (C_{\sigma} B)^{k-1} \\
\nonumber &\leq \big( \| x \|  +1 \big) \sum_{k=1}^\ell (C_{\sigma} B)^{k} + |\sigma(0)| \sum_{k=1}^\ell (C_{\sigma} B)^{k-1} \\
\nonumber &  \leq \big( \| x \|  +1 \big) \sum_{k=1}^L (C_{\sigma} B)^{k} + |\sigma(0)| \sum_{k=1}^L (C_{\sigma} B)^{k-1} \\
&\leq L \big(|\sigma(0)| + \| x \|  +1 \big) \max\big(1, (C_{\sigma} B)^L \big).        
\end{align}
Similarly, we have for all $\ell=1,\cdots,L$,
\begin{align*}
\| x^{(1)} - \widetilde{x}^{(1)} \| &= \|\sigma_1 \circ A_1(x) -  \sigma_1 \circ \widetilde{A}_1(x) \| \leq C_{\sigma} \| A_1(x) - \widetilde{A}_1(x) \| \\
& \leq C_{\sigma} \big( \|W_1 - \widetilde{W}_1 \| \| x\| + \|b_1 - \widetilde{b}_1 \|   \big) \leq C_{\sigma} \big(  \| x\| + 1 \big) \|\boldsymbol{\theta} - \widetilde{\boldsymbol{\theta}} \|.    
\end{align*} 
\begin{align} \label{eq_delta_ell_1}
\nonumber | x^{(\ell)} - \widetilde{x}^{(\ell)} | &= |\sigma_\ell \circ A_\ell(x^{(\ell-1)}) -  \sigma_\ell \circ \widetilde{A}_\ell(\widetilde{x}^{(\ell-1)}) | \leq C_{\sigma} \| A_\ell(x^{(\ell-1)}) - \widetilde{A}_\ell(\widetilde{x}^{(\ell-1)}) \| \\
\nonumber &\leq  C_{\sigma}  \|  (W_\ell - \widetilde{W}_\ell)x^{(\ell-1)} + \widetilde{W}_\ell (x^{(\ell-1)} - \widetilde{x}^{(\ell-1)} ) + (b_\ell - \widetilde{b}_\ell) \| \\
\nonumber & \leq   C_{\sigma} \big( \|  W_\ell - \widetilde{W}_\ell \| \|x^{(\ell-1)} \| + \| \widetilde{W}_\ell \|  \|x^{(\ell-1)} - \widetilde{x}^{(\ell-1)} \| + \| b_\ell - \widetilde{b}_\ell \| \big) \\
\nonumber &\leq  C_{\sigma} (\|x^{(\ell-1)} \| + 1) \| \boldsymbol{\theta} - \widetilde{\boldsymbol{\theta}} \|   +  C_{\sigma} B\|x^{(\ell-1)} - \widetilde{x}^{(\ell-1)} \| \\
 &\leq  2 C_{\sigma} L \big(|\sigma(0)| + \| x \|  +1 \big) \max\big(1, (C_{\sigma} B)^L \big) \| \boldsymbol{\theta} - \widetilde{\boldsymbol{\theta}} \|   +  C_{\sigma} B\|x^{(\ell-1)} - \widetilde{x}^{(\ell-1)} \|,
\end{align} 
where (\ref{eq_norm_x_ell}) is used in (\ref{eq_delta_ell_1}). This recurrence relation is also satisfied with $\ell=1$.
Hence, from (\ref{seq_recur}), we get for all $\ell = 1,\cdots,L$,
\begin{align}\label{eq_delta_ell_2}
 \nonumber | x^{(\ell)} - \widetilde{x}^{(\ell)} | & \leq  
    2 C_{\sigma} L \big(|\sigma(0)| + \| x \|  +1 \big) \max\big(1, (C_{\sigma} B)^L \big) \| \boldsymbol{\theta} - \widetilde{\boldsymbol{\theta}} \| \sum_{k=1}^\ell (C_{\sigma} B)^{k-1} \\
\nonumber & \leq 2 C_{\sigma} L \big(|\sigma(0)| + \| x \|  +1 \big) \max\big(1, (C_{\sigma} B)^L \big) \| \boldsymbol{\theta} - \widetilde{\boldsymbol{\theta}} \| \sum_{k=1}^L (C_{\sigma} B)^{k-1} \\
 &\leq 2 C_{\sigma} L^2 \big(|\sigma(0)| + \| x \|  +1 \big) \max\big(1, (C_{\sigma} B)^{2L} \big) \| \boldsymbol{\theta} - \widetilde{\boldsymbol{\theta}} \|.       
\end{align} 
Therefore, in addition to (\ref{eq_norm_x_ell}), we have,
\begin{align*}
 |h_{\boldsymbol{\theta}}(x) - h_{\widetilde{\boldsymbol{\theta}}} (x) |  &= | A_{L+1} \circ \sigma_L \circ A_L \circ \sigma_{L-1} \circ A_{L-1}\circ \cdots \circ \sigma_1 \circ A_1(x) - \widetilde{A}_{L+1} \circ \sigma_L \circ \widetilde{A}_L \circ \sigma_{L-1} \circ \widetilde{A}_{L-1}\circ \cdots \circ \sigma_1 \circ \widetilde{A}_1(x)  | \\
 &= | A_{L+1}(x^{(L)}) - \widetilde{A}_{L+1}(\widetilde{x}^{(L)})  | =  | W_{L+1}(x^{(L)}) - \widetilde{W}_{L+1}(\widetilde{x}^{(L)}) + b_{L+1} - \widetilde{b}_{L+1}  |  \\
 &\leq \|W_{L+1} - \widetilde{W}_{L+1} \| \| x^{(L)} \| + \| \widetilde{W}_{L+1} \| \| x^{(L)} - \widetilde{x}^{(L)} \| + \|b_{L+1} - \widetilde{b}_{L+1} \| \\
 &\leq (\| x^{(L)} \| + 1 ) \|\boldsymbol{\theta} - \widetilde{\boldsymbol{\theta}} \| + B \| x^{(L)} - \widetilde{x}^{(L)} \| \\
 & \leq 2 L \big(|\sigma(0)| + \| x \|  +1 \big) \max\big(1, (C_{\sigma} B)^L \big) \|\boldsymbol{\theta} - \widetilde{\boldsymbol{\theta}} \| +  2 C_{\sigma}B L^2 \big(|\sigma(0)| + \| x \|  +1 \big) \max\big(1, (C_{\sigma} B)^{2L} \big) \| \boldsymbol{\theta} - \widetilde{\boldsymbol{\theta}} \| \\
&\leq 2 L^2 \big(|\sigma(0)| + \| x \|  +1 \big) \big( 1 +  C_{\sigma} B \big) \max\big(1, (C_{\sigma} B)^{2L} \big) \| \boldsymbol{\theta} - \widetilde{\boldsymbol{\theta}} \|.    
\end{align*}
\end{dem}
%
 
 \subsection{Proof of Theorem \ref{theo_oracle_inq}}
 Let $\mathcal{I} \subseteq \{1,2,\cdots, n_{L,N}\}$ and $\eta \in (0,1]$. 
 Consider the probability measure $\rho:=\rho_{\mathcal{I}, \eta}$ on $\mathcal{S}_{\mathcal{I}}$ defined such that,
 \[ \dfrac{d \rho_{\mathcal{I},\eta } }{ d\Pi_{\mathcal{I}}(\boldsymbol{\theta})}  \propto \ind_{\| \theta - \boldsymbol{\theta}^*_{\mathcal{I}}  \| \leq \eta},  \]
 where $\boldsymbol{\theta}^*_{\mathcal{I}}$ is defined in (\ref{theta_star_I}).
 According to Lemma \ref{lem_lip_theta}, (\textbf{A2}), (\textbf{A3}) and the support of $\rho = \rho_{\mathcal{I},\eta }$, we have,
 \begin{align}\label{proof_int_E_theta_rho}
 \nonumber \int \mathcal{E}(\boldsymbol{\theta}) d \rho(\boldsymbol{\theta}) &= \mathcal{E}( \boldsymbol{\theta}^*_{\mathcal{I}} )  + \int \E_\pi \big[\ell\big(h_{\boldsymbol{\theta}}(\boldsymbol{X}_1),\boldsymbol{Y}_1 \big)  -  \ell\big(h_{\boldsymbol{\theta}^*_{\mathcal{I}}}(\boldsymbol{X}_1),\boldsymbol{Y}_1 \big) \big] d \rho(\boldsymbol{\theta}) \\
 \nonumber &= \mathcal{E}( \boldsymbol{\theta}^*_{\mathcal{I}} ) + C_{\ell} \int \E_\pi \big | h_{\boldsymbol{\theta}}(\boldsymbol{X}_1) - h_{\boldsymbol{\theta}^*_{\mathcal{I}}}(\boldsymbol{X}_1) \big | d \rho(\boldsymbol{\theta}) \\
 \nonumber &\leq  \mathcal{E}( \boldsymbol{\theta}^*_{\mathcal{I}} ) +  2 L^2 C_{\ell}  \big(|\sigma(0)| + \mk_{\mx}  +1 \big) \big( 1 +  C_{\sigma} B \big) \max\big(1, (C_{\sigma} B)^{2L} \big) \int \| \theta - \boldsymbol{\theta}^*_{\mathcal{I}} \| d \rho(\boldsymbol{\theta}) \\
  \nonumber &\leq  \mathcal{E}( \boldsymbol{\theta}^*_{\mathcal{I}} ) +  2 L^2 C_{\ell}  \big(|\sigma(0)| + \mk_{\mx}  +1 \big) \big( 1 +  C_{\sigma} B \big) \max\big(1, (C_{\sigma} B)^{2L} \big) \eta \\
  &\leq  \mathcal{E}( \boldsymbol{\theta}^*_{\mathcal{I}} ) + \dfrac{C_{\ell} }{n},       
 \end{align}
with
\[ \eta = \dfrac{1}{ 2 L^2  \big(|\sigma(0)| + \mk_{\mx}  +1 \big) \big( 1 +  C_{\sigma} B \big) \max\big(1, (C_{\sigma} B)^{2L} \big) n}   .\]
Now, consider the Kullback-Leibler term in (\ref{PAC_bound_KL}). 
From Lemma 10 in \cite{steffen2022pac}, we have,
 \begin{equation}\label{proof_KL_bound}
  \KL(\rho, \Pi) = \KL(\rho_{\mathcal{I}, \eta}, \Pi) \leq |\mathcal{I}| \log\Big( \dfrac{ 2 s C_s B n_{L,N} e }{ \eta}\Big),
 \end{equation}
where $s \geq 2$ denotes sparsity parameter and $C_s$ is defined in (\ref{def_Pi}).
%
%
So, from (\ref{PAC_bound_KL}), (\ref{proof_int_E_theta_rho}) and (\ref{proof_KL_bound}), we get with probability at least $1-\delta$,
 \begin{align*} 
   \mathcal{E}(\widehat{\boldsymbol{\theta}}_\lambda) &\leq  \dfrac{1}{  1 - \dfrac{16 K \lambda}{  n \gamma_{n} - 10 \lambda  }    }   \Bigg[ \Big( 1 + \dfrac{16 K \lambda}{ n \gamma_{n} - 10 \lambda} \Big)  \Big( \mathcal{E}( \boldsymbol{\theta}^*_{\mathcal{I}} ) + \dfrac{C_{\ell} }{n} \Big) + \dfrac{2}{\lambda} \Big[ |\mathcal{I}| \log\Big( \dfrac{ 2 s C_s B n_{L,N} e }{ \eta}\Big)  + \log(1/\delta) \Big]   \Bigg],
 \end{align*}  
By choosing,
\[   \lambda =\dfrac{n \gamma_{n}}{ 32K + 10}, \] 
we have with probability at least $1-\delta$,
 \begin{align} \label{proof_E_theta_Xi}
 \nonumber  \mathcal{E}(\widehat{\boldsymbol{\theta}}_\lambda) &\leq  3 \mathcal{E}( \boldsymbol{\theta}^*_{\mathcal{I}} ) +  \dfrac{3 C_{\ell} }{n}  +    \dfrac{4 ( 32K + 10) }{ n \gamma_{n}} \Big[ |\mathcal{I}| \log\Big( \dfrac{ 2 s C_s B n_{L,N} e }{ \eta}\Big)  + \log(1/\delta) \Big] \\
   &\leq  3 \mathcal{E}( \boldsymbol{\theta}^*_{\mathcal{I}} ) +  \dfrac{ \Xi_1 }{n \gamma_{n}} \Big[|\mathcal{I}| L \log\big( \max(n, B, n_{L,N} )  \big)   +  \log(1/\delta) \Big] +  \dfrac{3 C_{\ell} }{n}
\\
   &\leq  3 \mathcal{E}( \boldsymbol{\theta}^*_{\mathcal{I}} ) +  \dfrac{ \Xi_1 }{n \gamma_{n}} \Big[|\mathcal{I}| L \log\big( \max(n, B, n_{L,N} )  \big)   +  \log(1/\delta) + C_{\ell} \Big]
 \end{align}    
where $\Xi_1$ is a constant which is independent of $n, L, N, B, F, \gamma_{n}, C_{\ell}$, and where we used that $\gamma_n\leq 1$.
This completes the proof of the theorem, since the bound in (\ref{proof_E_theta_Xi}) holds for all $\mathcal{I} \subseteq \{1,2,\cdots, n_{L,N}\}$.
\qed

 \subsection{Proof of Theorem \ref{theo_bound_Holder}}
  %
%
%
We have,
\begin{align*}
& \min_{ \mathcal{I} \subseteq \{1,2,\cdots, n_{L_n,N_n}\} } \Bigg(  3 \mathcal{E}( \boldsymbol{\theta}^*_{\mathcal{I}} ) +  \dfrac{ \Xi_1 }{n \gamma_{n}} \Big[|\mathcal{I}| L_n \log\big( \max(n,  B_n, n_{L_n,N_n} )  \big)   +  \log(1/\delta) \Big]    \Bigg)  \\
& =     \min_{ \mathcal{I} \subseteq \{1,2,\cdots, n_{L_n,N_n}\} } \Bigg(  3 \inf_{\theta \in \mathcal{S}_{\mathcal{I}} } \big( R(h_{\boldsymbol{\theta}}) - R(h^*)\big)  +  \dfrac{ \Xi_1 }{n \gamma_{n}} \Big[|\mathcal{I}| L_n \log\big( \max(n, B_n, n_{L_n,N_n} )  \big)   +  \log(1/\delta) \Big]    \Bigg) ,
\end{align*}  
where the constant $\Xi_1$ is given in Theorem \ref{theo_oracle_inq}.
In the sequel, we set:
\begin{equation}\label{proof_def_H_Sigma_n}
\mathcal{H}_{\sigma,\dx,\dy, n} := \mathcal{H}_{\sigma,\dx,\dy} (L_n,N_n,  B_n, F_n, S_n),
\end{equation}
see (\ref{def_H_lnbfs}). 
Let $h \in \mathcal{H}_{\sigma,\dx,\dy, n}$.
Set $\boldsymbol{\theta}(h)= (\theta_1,\cdots,\theta_{n_{L_n,N_n}})$ and $\mathcal{J} = \big\{ i \in \{1,\cdots, n_{L_n,N_n}\}, ~ \theta_i\neq 0  \big\}$.
We have, $ \mathcal{J} \subset \{1,2,\cdots, n_{L_n,N_n}\} $, $\boldsymbol{\theta}(h) \in \mathcal{S}_{\mathcal{J}}$ and $|\mathcal{J}| = |\boldsymbol{\theta}(h)|_0 \leq S_n$. 
Hence,
\begin{align}\label{proof_3K_ell_E_h_star_S_n}
\nonumber &  \min_{ \mathcal{I} \subseteq \{1,2,\cdots, n_{L_n,N_n}\} } \Bigg(  3 \inf_{\theta \in \mathcal{S}_{\mathcal{I}} } \big( R(h_{\boldsymbol{\theta}}) - R(h^*)\big)  +  \dfrac{ \Xi_1 }{n \gamma_{n}} \Big[|\mathcal{I}| L_n \log\big( \max(n,  B_n, n_{L_n,N_n} )  \big)   +  \log(1/\delta) \Big]    \Bigg) \\
\nonumber & \leq   3 \inf_{\theta \in \mathcal{S}_{\mathcal{J}} } \big( R(h_{\boldsymbol{\theta}}) - R(h^*)\big)  +  \dfrac{ \Xi_1 }{n \gamma_{n}} \Big[|\mathcal{J}| L_n \log\big( \max(n, B_n, n_{L_n,N_n} )  \big)   +  \log(1/\delta) \Big]  \\
  & \leq   3  \big( R(h)  - R(h^*)\big)  +  \dfrac{ \Xi_1 }{n \gamma_{n}} \Big[S_n  L_n \log\big( \max(n, B_n, n_{L_n,N_n} )  \big)   +  \log(1/\delta) \Big]  
%
%
\end{align} 
%
%
Since (\ref{proof_3K_ell_E_h_star_S_n}) is satisfied for all $h \in \mathcal{H}_{\sigma,\dx,\dy, n}$, we get,
\begin{align*} 
\nonumber &  \min_{ \mathcal{I} \subseteq \{1,2,\cdots, n_{L_n,N_n}\} } \Bigg(  3 \inf_{\theta \in \mathcal{S}_{\mathcal{I}} } \big( R(h_{\boldsymbol{\theta}}) - R(h^*)\big)  +  \dfrac{ \Xi_1 }{n \gamma_{n}} \Big[|\mathcal{I}| L_n \log\big( \max(n, B_n, n_{L_n,N_n} )  \big)   +  \log(1/\delta) \Big]    \Bigg) \\
 & \leq  \inf_{h \in \mathcal{H}_{\sigma,\dx,\dy, n}} \Bigg(  3  \big( R(h)  - R(h^*)\big)   +  \dfrac{ \Xi_1 }{n \gamma_{n}} \Big[S_n  L_n \log\big( \max(n, B_n, n_{L_n,N_n} )  \big)   +  \log(1/\delta) \Big] \Bigg) \\
 & \leq  3 \inf_{h \in \mathcal{H}_{\sigma,\dx,\dy, n}}   \big( R(h)  - R(h^*)\big)   +  \dfrac{ \Xi_1 }{n \gamma_{n}} \Big[S_n  L_n \log\big( \max(n, B_n, n_{L_n,N_n} )  \big)   +  \log(1/\delta) \Big] .  
\end{align*} 
Therefore, from Theorem \ref{theo_oracle_inq}, it holds with $P_n \otimes \Pi_\lambda$-probability at least $1-\delta$ that,
 \begin{equation}\label{proof_E_theta_lambda_inf_h_3k_ell}
   \mathcal{E}(\widehat{\boldsymbol{\theta}}_\lambda) \leq 3 \inf_{h \in \mathcal{H}_{\sigma,\dx,\dy, n}}   \big( R(h)  - R(h^*)\big)   +  \dfrac{ \Xi_1 }{n \gamma_{n}} \Big[S_n  L_n \log\big( \max(n, B_n, n_{L_n,N_n} )  \big)   +  \log(1/\delta) +  C_{\ell}  \Big] .
 \end{equation} 
{\color{black}
Set $\tilde{n}=n\gamma_n$ the ``effective sample size'', and $\varepsilon_n = \dfrac{1}{(\tilde{n})^{\frac{\beta}{2\beta+d_x}}}$.
Since $h^* \in \mathcal{C}^{\beta, \mathcal{K}}(\mx)$ for some $s,\mathcal{K}>0$, then from \cite{kengne2025excess} (in the case of ReLU DNN, see also Theorem 1 in \cite{schmidt2019deep}), there exist constants $L_0 , N_0 , S_0 , B_0 > 0$ such that with 
\begin{equation}\label{proof_architecture}
 L_n = \dfrac{\beta L_0}{\beta+ d_x} \log \tilde{n}, ~ N_n=N_0 \tilde{n}^{\frac{d_x}{2\beta+d_x}}, ~ S_n=\dfrac{\beta S_0 }{2\beta + d_x} \tilde{n}^{\frac{d_x}{2\beta+d_x}} \log \tilde{n}, ~ B_n = B_0 \tilde{n}^{\frac{4(\beta+\dx)}{2\beta+\dx}},
\end{equation}
there is a neural network $h_n \in \mathcal{H}_{\sigma,\dx,\dy, n}$ that satisfies,
\[ \|h_n - h^*\|_{\infty, \mx}  \leq \varepsilon_n  .\] 
Set,
\[   \widetilde{\mathcal{H}}_{\sigma,\dx,\dy, n} := \{h \in  \mathcal{H}_{\sigma,\dx,\dy, n}, ~ \|h - h^*\|_{\infty, \mx}  \leq \varepsilon_n\}  .\]
Also, recall that $n_{L_n,N_n} = (L_n+1) N^{2}_n  +  L_n N_n \leq  2(L_n+1) N^{2}_n$. 
Hence, from (\ref{proof_E_theta_lambda_inf_h_3k_ell}) and in addition to (\textbf{A7}), we get for $\tilde{n} \geq \varepsilon_0^{-(2 + \dx/s) }$,
 \begin{align}\label{proof_inq_E_theta_lambda_Xi_2}
  \nonumber  \mathcal{E}(\widehat{\boldsymbol{\theta}}_\lambda) &\leq  3 \inf_{h \in \widetilde{\mathcal{H}}_{\sigma,\dx,\dy, n}}   \big( R(h)  - R(h^*)\big)   +  \dfrac{ \Xi_1 }{\tilde{n}} \Big[S_n  L_n \log\big( \max(n, B_n, n_{L_n,N_n} )  \big)   +  \log(1/\delta) +   C_{\ell}  \Big] \\
 \nonumber &\leq  3  \mk_0 \inf_{h \in \widetilde{\mathcal{H}}_{\sigma,\dx,\dy, n}} \| h - h^*\|^2_{2, P_{\boldsymbol{X}_0}}   +  \dfrac{ \Xi_1 }{\tilde{n}} \Big[S_n  L_n \log\big( \max(n, B_n, n_{L_n,N_n} )  \big)   +  \log(1/\delta) +   C_{\ell} \Big] \\
  \nonumber &\leq  3  \mk_0 \inf_{h \in \widetilde{\mathcal{H}}_{\sigma,\dx,\dy, n}} \|h_n - h^*\|_{\infty, \mx}^2   +  \dfrac{ \Xi_1 }{\tilde{n}} \Big[S_n  L_n \log\big( \max(n, B_n, n_{L_n,N_n} )  \big)   +  \log(1/\delta) +   C_{\ell} \Big] \\
   \nonumber &\leq  3  \mk_0 \varepsilon_n^2   +  \dfrac{ \Xi_1 L_0 S_0 \frac{\beta^2 }{(\beta+ d_x)(2\beta + \dx)} \tilde{n}^{\frac{d_x}{2\beta+d_x}} (\log \tilde{n})^2  \log\Big( \max\Big[n, B_0 \tilde{n}^{\frac{4(\beta+\dx)}{2\beta+\dx} }, 2\big(\frac{s L_0}{\beta+ d_x} \log \tilde{n} +1 \big) N_0^2 \tilde{n}^{\frac{2 d_x}{2\beta+d_x}} \Big]  \Big) }{\tilde{n}}
   \\
   \nonumber
   & \quad \quad \quad + \dfrac{ \Xi_1 [\log(1/\delta)+   C_{\ell}] }{\tilde{n}}  \\
 &\leq \Xi_2 \Bigg( \dfrac{ \log^3 \tilde{n} }{\tilde{n}^{\frac{2\beta}{2\beta+d_x}}}     +  \dfrac{   \log(1/\delta) + C_{\ell}}{\tilde{n} } \Bigg)
 = \Xi_2 \Bigg( \dfrac{ \log^3 ( n\gamma_n ) }{(n\gamma_n)^{\frac{2\beta}{2\beta+d_x}}}     +  \dfrac{   \log(1/\delta) + C_{\ell}}{n\gamma_n} \Bigg),
\end{align}
}
for some constant $\Xi_2 >0$, independent of $n$ and $C_{\ell}$.
\qed

 

 \subsection{Proof of Theorem \ref{theo_bound_comp}}
 Set $\mathcal{H}_{\sigma,\dx,\dy, n} := \mathcal{H}_{\sigma,\dx,\dy} (L_n,N_n,  B_n, F_n, S_n)$ and 
\begin{equation*}
\widetilde{\mathcal{H}}_{\sigma,\dx,\dy, n}^{(1)} := \big\{ h \in \mathcal{H}_{\sigma,\dx,\dy, n}, \text{ with architecture }  (L_n, d_x,N_n,\cdots,N_n, d_y) \big\}. 
\end{equation*}  
Let $\delta>0$. {\color{black} As in the previous proof, we put $\tilde{n} = n\gamma_n$.
According to (\ref{proof_E_theta_lambda_inf_h_3k_ell}), we get with $P_n \otimes \Pi_\lambda$-probability at least $1-\delta$,
\begin{equation}\label{proof_E_theta_lambda_inf_H_tilde}
   \mathcal{E}(\widehat{\boldsymbol{\theta}}_\lambda) \leq  3 \inf_{h \in \widetilde{\mathcal{H}}_{\sigma,\dx,\dy, n}^{(1)}}   \big( R(h)  - R(h^*)\big)    +  \dfrac{ \Xi_1 }{\tilde{n} } \Big[S_n  L_n \log\big( \max(n, B_n, n_{L_n,N_n} )  \big)   +  \log(1/\delta) + C_{\ell} \Big]   .
 \end{equation} 
Since $h^* \in \mathcal{G}(q, \boldsymbol{d}, \boldsymbol{t}, \boldsymbol{\beta}, \mathcal{K})$, from the proof of Theorem 1 in \cite{schmidt2020nonparametric}, one can find a neural network $h_n \in \widetilde{\mathcal{H}}_{\sigma,\dx,\dy, n}^{(1)}$ satisfying 
\begin{equation*}
 \|h_n -  h^* \|_{\infty, \mx}^2 \leq C_1 \max_{i=0,\cdots,q}  \tilde{n}^{- \frac{2\beta_i^*}{2 \beta_i^* + t_i}} = C_1 \phi_{\tilde{n}},
\end{equation*}  
for some constant $C_1 >0$, independent of $n$.
Set,
\begin{equation*}
\widetilde{\mathcal{H}}_{\sigma,\dx,\dy, n}^{(2)} := \big\{ h \in \widetilde{\mathcal{H}}_{\sigma,\dx,\dy, n}^{(1)}, ~ \|h_n -  h^* \|_{\infty, \mx} \leq \sqrt{\phi_{\tilde{n}}} \big\}.
\end{equation*}
In addition to the architecture parameters in (\ref{theo_upper_cond_architecture}), (\ref{proof_E_theta_lambda_inf_H_tilde}) gives with $P_n \otimes \Pi_\lambda$-probability at least $1-\delta$,
\begin{align*}
 \mathcal{E}(\widehat{\boldsymbol{\theta}}_\lambda) & \leq  3 \inf_{h \in \widetilde{\mathcal{H}}_{\sigma,\dx,\dy, n}^{(2)}}   \big( R(h)  - R(h^*)\big)    +  \dfrac{ \Xi_1 }{\tilde{n}} \Big[S_n  L_n \log\big( \max(n, B_n, n_{L_n,N_n} )  \big)   +  \log(1/\delta) + C_{\ell}  \Big] \\
  & \leq  3 \mk_0 \inf_{h \in \widetilde{\mathcal{H}}_{\sigma,\dx,\dy, n}^{(2)}}  \| h - h^*\|^2_{2, P_{\boldsymbol{X}_0}}    +  \dfrac{ \Xi_1 }{n \gamma_{n}} \Big[S_n  L_n \log\big( \max(n, B_n, n_{L_n,N_n} )  \big)   +  \log(1/\delta) + C_{\ell} \Big] \\
  &\leq   3 \mk_0 C_1 \phi_{\tilde{n}}  +   \dfrac{ \Xi_1 }{\tilde{n}} \Big[C_2 \tilde{n}\phi_{\tilde{n}}  (\log \tilde{n})^3  +  \log(1/\delta) + C_{\ell} \Big]    \\
 &\leq    3 \mk_0 C_1 \phi_{\tilde{n}}  + \Xi_1  C_2 \phi_{\tilde{n}}  (\log \tilde{n})^3   +  \dfrac{\Xi_1 [\log(1/\delta)+ C_{\ell} ]}{\tilde{n}}   \\
  &\leq     3 \mk_0 C_1 \phi_n   + \Xi_1  C_2 \phi_{\tilde{n}} (\log \tilde{n})^3  +  \dfrac{\Xi_1 [\log(1/\delta)+ C_{\ell} ]}{\tilde{n}}   \\
 &\leq   \Xi_3  \Big( \phi_{\tilde{n}} (\log \tilde{n})^3 + \dfrac{\log(1/\delta) + C_{\ell}}{\tilde{n} } \Big),
\end{align*}
}
for some constants $C_2, \Xi_3 >0$ independent of $n$ and $C_{\ell}$.
\qed

\subsection{Proof of Theorem \ref{theo_lower_bound_class}}
Since $\mathcal{G}(q, \bold{d}, \bold{t}, \boldsymbol{\beta}, \mk)$ is a class of bounded functions, it suffices to establish (\ref{equa_lower_bound}) with the infimum taken over the class of bounded estimators.
Let $\mk_1 >0$. Consider a target function $h^{*}\in\mathcal{G}(q, \bold{d}, \bold{t}, \boldsymbol{\beta}, \mk)$ and a predictor $h : \mx \rightarrow \R$ satisfying $\|h\| \leq \mk_1$.
Let $\boldsymbol{x} \in \mx$.
We have 
\begin{align}\label{cond_excess_risk_class_lower_b}
\nonumber & \E\big[ \phi\big(Y_0 h(\boldsymbol{X}_0)  \big) | \boldsymbol{X}_0=\boldsymbol{x} \big] - \E\big[ \phi\big(Y_0 h^*(\boldsymbol{X}_0)  \big) | \boldsymbol{X}_0=\boldsymbol{x} \big] \\
\nonumber &= \eta(\boldsymbol{x}) \phi\big( h(\boldsymbol{x}) \big) + \big(1- \eta(\boldsymbol{x}) \big) \phi\big( -h(\boldsymbol{x}) \big) - \eta(\boldsymbol{x}) \phi\big( h^*(\boldsymbol{x}) \big) - \big(1- \eta(\boldsymbol{x}) \big) \phi\big( -h^*(\boldsymbol{x}) \big) \\
  &= \eta(\boldsymbol{x}) \big(\phi\big( h(\boldsymbol{x}) \big) - \phi\big( h^*(\boldsymbol{x}) \big)  \big) + \big(1- \eta(\boldsymbol{x}) \big) \big( \phi\big( -h(\boldsymbol{x}) \big) - \phi\big( -h^*(\boldsymbol{x}) \big) \big).
\end{align}
Set $\mk_2 := \max(\mk,\mk_1)$.
Since the function $\phi$ is strongly convex on $[-\mk_2, \mk_2]$, there exists a constant $C_3:=C_3(\mk_2) >0$ such that for all $y, y' \in [-\mk_2, \mk_2]$,
\[ \phi(y) - \phi(y') \geq (y-y')\phi'(y') +  C_3(y-y')^2 ,  \] 
where $\phi'(y) = \partial \phi (y)/\partial y$.
Hence, (\ref{cond_excess_risk_class_lower_b}) gives
\begin{align}\label{cond_excess_risk_class_lower_ing}
\nonumber & \E\big[ \phi\big(Y_0 h(\boldsymbol{X}_0)  \big) | \boldsymbol{X}_0=\boldsymbol{x} \big] - \E\big[ \phi\big(Y_0 h^*(\boldsymbol{X}_0)  \big) | \boldsymbol{X}_0=\boldsymbol{x} \big] \\
\nonumber &\geq  \eta(\boldsymbol{x}) \big[ \big(  h(\boldsymbol{x}) -  h^*(\boldsymbol{x}) \big) \phi'\big( h^*(\boldsymbol{x}) \big)  + C_3 \big(  h(\boldsymbol{x}) -  h^*(\boldsymbol{x})\big)^2 \big] +  \big(1- \eta(\boldsymbol{x}) \big) \big[ - \big(  h(\boldsymbol{x}) -  h^*(\boldsymbol{x}) \big) \phi'\big(- h^*(\boldsymbol{x}) \big)  + C_3 \big(  h(\boldsymbol{x}) -  h^*(\boldsymbol{x})\big)^2 \big] \\
&\geq \big(  h(\boldsymbol{x}) -  h^*(\boldsymbol{x}) \big) \big[ \eta(\boldsymbol{x}) \phi'\big( h^*(\boldsymbol{x}) \big) - \big(1- \eta(\boldsymbol{x}) \big) \phi'\big(- h^*(\boldsymbol{x}) \big)   \big] + C_3 \big(  h(\boldsymbol{x}) -  h^*(\boldsymbol{x})\big)^2.
\end{align}
Consider the function $\psi : [-\mk_2, \mk_2] \rightarrow \R$, defined by,
\begin{equation}
\psi(\alpha) = \eta(\boldsymbol{x}) \phi\big( \alpha \big) + \big(1- \eta(\boldsymbol{x}) \big) \phi\big( -\alpha \big).
\end{equation}
Recall that $h^*$ is a target function. We get,
\[ h^*(\boldsymbol{x}) \in \argmin_{\alpha \in [-\mk_2, \mk_2]} \psi(\alpha)   .\]
So, since $\psi$ is a convex function, it holds that $\psi'\big( h^*(\boldsymbol{x}) \big) = 0$.
Therefore, we have from (\ref{cond_excess_risk_class_lower_ing}), 
\begin{equation*} 
 \E\big[ \phi\big(Y_0 h(\boldsymbol{X}_0)  \big) | \boldsymbol{X}_0=\boldsymbol{x} \big] - \E\big[ \phi\big(Y_0 h^*(\boldsymbol{X}_0)  \big) | \boldsymbol{X}_0=\boldsymbol{x} \big] \geq C_3 \big(  h(\boldsymbol{x}) -  h^*(\boldsymbol{x})\big)^2.
\end{equation*} 
That is,
\begin{equation}\label{cond_excess_risk_class_lower_ing2}
 R(h) - R(h^*) \geq C_3 \| h - h^*\|^2_{2, P_{\boldsymbol{X}_0}}.
\end{equation}
Thus, to prove the theorem, it suffices to establish that,
there exist a constant $C>0$  such that 
\begin{equation}\label{proof_equa_lower_bound}
\underset{\widehat{h}_n}{\inf}~\underset{h^{*}\in\mathcal{G}(q, \bold{d}, \bold{t}, \boldsymbol{\beta}, \mk)}{\sup} \E\big[ \| \widehat{h}_n - h^*\|^2_{2, P_{\boldsymbol{X}_0}} \big]  \ge C\phi_{n}.
\end{equation}
We will establish (\ref{proof_equa_lower_bound}) when the observations $(\boldsymbol{X}_1,Y_1),\cdots, (\boldsymbol{X_n},Y_n)$ are i.i.d. and deduce such lower bound in the dependent case.
To do so, we will apply the Theorem 2.7 in \cite{Tsybakov2009}.

\medskip

Consider a class of composition structured functions $\mathcal{G}(q, \bold{d}, \bold{t}, \boldsymbol{\beta}, A)$.
From the proof of Theorem 3 in \cite{schmidt2020nonparametric}, there exists an integer $M \ge 1$, a constant $\kappa >0$ depending only on $\bold{t}$ and $\bold{\beta}$, and functions $h_{(0)}, \dots, h_{(M)} \in \mathcal{G}(q, \bold{d}, \bold{t}, \boldsymbol{\beta}, A)$ such that, 
\begin{equation} \label{proof_h_j_h_k_kappa}
  \|h_{(j)} - h_{(k)}\|_{2}^2 \ge \kappa^2 \phi_{n}, ~ \forall~ 0\leq j < k \leq M. 
\end{equation}
Therefore, the first item of Theorem 2.7 in \cite{Tsybakov2009} is satisfied.

\medskip

 Consider model (\ref{mod_class}) with the target function $h^*$ in (\ref{target_class}) and denote by $P_{h^*}^{\otimes n}$ the law of the sample $D_n$.
 For all $j=0,\cdots,M$, set $P_j^{\otimes n} := P_{h_{(j)}}^{\otimes n}$ where $h_{(j)}$ satisfies (\ref{target_class}) with a parameter function,
 \[ \eta_j(\boldsymbol{x}) = \mathds{P}(Y_t = 1| \boldsymbol{X}_t = \boldsymbol{x}), ~ \forall \boldsymbol{x} \in \mx  \]
 and denote by $P_j$ the distribution of $(\boldsymbol{X}_i,\boldsymbol{Y}_i)$ (for any $i=1,\cdots,n$), generated from (\ref{mod_class}) for $h^* = h_{(j)} $ in (\ref{target_class}).
From Lemma C.19 and the proof of Lemma C.20 in \cite{zhang2024classification}, we get for any $j=1,\cdots,M$,
\begin{align}\label{proof_Kullback_P_j_P_0}
\nonumber &\KL(P_j, P_0) = \int_{\mx} \Big[ \eta_j(\boldsymbol{x}) \log\Big(
\dfrac{\eta_j(\boldsymbol{x})}{\eta_0(\boldsymbol{x})} \Big) + \big( 1 - \eta_j(\boldsymbol{x}) \big) \log\Big(\dfrac{1 - \eta_j(\boldsymbol{x})}{1 - \eta_0(\boldsymbol{x})} \Big) \Big]dP_{\boldsymbol{X}_0}(x) \\
\nonumber
 &= \int_{\mx} \Big[ \eta_j(\boldsymbol{x}) \log\Big(
\dfrac{1}{\eta_0(\boldsymbol{x})} \Big)  + \big( 1 - \eta_j(\boldsymbol{x}) \big) \log\Big(\dfrac{1}{1 - \eta_0(\boldsymbol{x})} \Big) -\eta_j(\boldsymbol{x}) \log\Big(
\dfrac{1}{\eta_j(\boldsymbol{x})} \Big)  -  \big( 1 - \eta_j(\boldsymbol{x}) \big) \log\Big(\dfrac{1}{1 - \eta_j(\boldsymbol{x})} \Big) \Big]dP_{\boldsymbol{X}_0}(x) \\
\nonumber
  &= \int_{\mx} \Big[ \eta_j(\boldsymbol{x}) \phi\big( h_{(0)}(\boldsymbol{x}) \big)  + \big( 1 - \eta_j(\boldsymbol{x}) \big)\phi\big( -h_{(0)}(\boldsymbol{x}) \big)  -\eta_j(\boldsymbol{x}) \log\Big(
\dfrac{1}{\eta_j(\boldsymbol{x})} \Big)  -  \big( 1 - \eta_j(\boldsymbol{x}) \big) \log\Big(\dfrac{1}{1 - \eta_j(\boldsymbol{x})} \Big) \Big]dP_{\boldsymbol{X}_0}(x) \\
&\leq \dfrac{1}{8} \int_{\mx} |h_{(0)}(\boldsymbol{x}) - h_{(j)}(\boldsymbol{x}) |^2 dP_{\boldsymbol{X}_0}(\boldsymbol{x}) \leq \dfrac{1}{8} \|h_{(j)} - h_{(0)}\|_{2}^2, 
\end{align}
where the inequality in (\ref{proof_Kullback_P_j_P_0}) is obtained by using Lemma C.6 in \cite{zhang2024classification}, see also (\ref{cond_excess_risk_class}).
Note that, from the proof of Theorem 3 in \cite{schmidt2020nonparametric}, the functions $h_{(0)}, \dots, h_{(M)} \in \mathcal{G}(q, \bold{d}, \bold{t}, \boldsymbol{\beta}, A)$ satisfying (\ref{proof_h_j_h_k_kappa}) can be constructed so that,
\begin{equation}\label{proof_h_j_h_0_sum_M}
n \sum_{j=1}^M \|h_{(j)} - h_{(0)}\|_{2}^2 \leq \dfrac{8 M}{9} \log M . 
\end{equation}
According to (\ref{proof_Kullback_P_j_P_0}) and (\ref{proof_h_j_h_0_sum_M}), we have (in the i.i.d. case)
\begin{equation}\label{proof_KL_sum_M}
\dfrac{1}{M} \sum_{j=1}^M \KL(P_j^{\otimes n}, P_0^{\otimes n}) = \dfrac{n}{M} \sum_{j=1}^M \KL(P_j, P_0)  \leq \dfrac{1}{9} \log M . 
\end{equation}
Thus, the second item of Theorem 2.7 in \cite{Tsybakov2009} is satisfied.
Hence, one can find  a constant $C>0$  such that (\ref{proof_equa_lower_bound}) holds.
This completes the proof of the theorem.
\qed

 \medskip

\medskip

\section*{Acknowledgements}
 The authors are very grateful to the two anonymous Referees for many relevant suggestions and comments which helped to improve the contents of this paper.

\bibliographystyle{acm}

\end{document}